\documentclass[12pt,a4paper]{article}
\usepackage{microtype}
\usepackage{graphicx}
\usepackage{verbatim}
\usepackage{xcolor}
\usepackage{subcaption}
\usepackage{hyperref}

\usepackage{soul} % For \st{} (strikethrough text)

\usepackage{empheq}

\usepackage{amsmath}
\usepackage{amssymb}
\usepackage{mathtools}
\usepackage{amsthm}
\usepackage{thmtools}
\usepackage{cancel}
\usepackage[T1]{fontenc}
\usepackage[numbers]{natbib}
\usepackage{cleveref}
\crefformat{section}{\S\hspace{0pt}\textnormal{#2#1#3}}
\crefformat{theorem}{Theorem~\textnormal{#2#1#3}}
\crefformat{corollary}{Corollary~\textnormal{#2#1#3}}
\crefformat{definition}{Definition~\textnormal{#2#1#3}}
\crefformat{figure}{Fig.~\textnormal{#2#1#3}}
\crefformat{example}{Example~\textnormal{#2#1#3}}
\numberwithin{equation}{section}

\usepackage{makecell}
\usepackage{tikz-cd}

\newcommand{\id}{\mathrm{id}}

\newcommand{\SO}{\mathrm{SO}}

\newcommand{\R}{\mathbb{R}}

\newcommand{\Z}{\mathbb{Z}}

\theoremstyle{plain}
\newtheorem{theorem}{Theorem}[section]

\newtheorem{lemma}[theorem]{Lemma}

\theoremstyle{definition}
\newtheorem{definition}[theorem]{Definition}
\newtheorem{example}[theorem]{Example}

\theoremstyle{remark}
\newtheorem{remark}[theorem]{Remark}

\usepackage{enumitem,amssymb}
\newlist{todolist}{itemize}{2}
\setlist[todolist]{label=$\square$}
\usepackage{pifont}
\title{Steerable Neural ODEs on Homogeneous Spaces
}
\date{}

\begin{document}

\maketitle

\vspace*{-10mm}
\begin{center}
\begin{minipage}[t]{0.8\textwidth}
\textbf{Emma Andersdotter} \hfill {\footnotesize \tt emma.andersdotter@umu.se}\\
\textit{\footnotesize Department of Mathematics and Mathematical Statistics\\
Ume{\aa} University\\ Ume{\aa} SE-901 87, Sweden}
\end{minipage}
\end{center}

\begin{center}
\begin{minipage}[t]{0.8\textwidth}
\textbf{Daniel Persson} \hfill {\footnotesize \tt daniel.persson@chalmers.se} \\
\textit{\footnotesize Department of Mathematical Sciences\\ Chalmers University of Technology and Gothenburg University\\Gothenburg SE-412 96, Sweden}
\end{minipage}
\end{center}

\begin{center}
\begin{minipage}[t]{0.8\textwidth}
\textbf{Fredrik Ohlsson}  \hfill {\footnotesize \tt fredrik.ohlsson@umu.se}\\
\textit{\footnotesize Department of Mathematics and Mathematical Statistics\\
Ume{\aa} University\\ Ume{\aa} SE-901 87, Sweden}
\end{minipage}
\end{center}
\vspace*{5mm}

\begin{abstract}
\noindent We introduce \textit{steerable neural ordinary differential equations} on homogeneous spaces $M=G/H$. These models constitute a novel geometric extension of manifold neural ordinary differential equations (NODEs) that transport associated feature vectors transforming under the local symmetry group $H$.
We interpret features as sections of associated vector bundles over $M$, and describe their evolution as parallel transport. This results in a coupled system of ODEs consisting of a flow equation on $M$ and a steering equation acting on features. 
We show that steerable NODEs are $G$-equivariant whenever the vector field generating the flow and the connection governing parallel transport are both $G$-invariant. Furthermore, we demonstrate how steerable NODEs incorporate existing NODE models and continuous normalizing flows on Lie groups. Our framework provides the geometric foundation for learning continuous-time equivariant dynamics of general vector-valued features on homogeneous spaces.
\end{abstract}

\newpage
\tableofcontents

\section{Introduction}\label{sec:introduction}
\subsection{Background and motivation}
\label{sec:intro_background}
Neural ordinary differential equations (NODEs), introduced by Chen et al.~\cite{Chen2018}, are a class of machine learning models in which data is transformed continuously in time according to a learnable vector field, rather than through a finite sequence of discrete layers as in traditional neural networks. Formally, a NODE defines a dynamical system whose flow map transports input data to output data by integrating an ordinary differential equation. This continuous formulation endows NODEs with several attractive properties, including reversibility, parameter sharing across depth, and a natural interpretation as diffeomorphisms of the data space. In addition, NODEs provide a principled way of modelling continuous-time dynamics and are particularly well suited for problems involving irregular sampling, adaptive computation, and dynamical systems.

One of the most influential applications of NODEs is in generative modelling, where they give rise to continuous normalizing flows (CNFs)~\cite{Chen2018}. In this setting, a NODE defines a diffeomorphic change of variables, and probability densities are transformed via the induced push-forward along the flow. CNFs combine the expressivity of deep generative models with exact likelihood evaluation and invertibility, and have been further developed in a number of geometric directions, including group manifolds and quotient spaces. This includes models for generating molecular and protein structures~\cite{Yim2023,Bose2024,Li2024,Geffner2025,Dunn2026}, and applications to lattice gauge theories and physics-inspired models~\cite{Gerdes2024,huang2025learning}.

Many modern machine learning problems involve data that are inherently non-Euclidean and exhibit symmetries, such as data defined on manifolds, graphs, or quotient spaces.  Geometric deep learning aims to incorporate such structures directly into model architectures which respect the underlying geometry and symmetry constraints. From this perspective, neural networks operating on $\mathbb{R}^n$ are a special case, and more general constructions are required to handle data living on curved spaces or transforming under group actions. Lie groups and homogeneous spaces play a central role in this setting, as they provide a natural language for describing continuous symmetries encountered in physics, robotics, and machine learning.

A natural setting for NODEs in geometric deep learning is to consider ODEs on a smooth manifold $M$, rather than $\mathbb{R}^n$, to accommodate non-Euclidean data~\cite{Falorsi2020,Lou2020,Mathieu2020}. The ODE that governs the NODE model is then given by the Cauchy problem
\begin{equation}
\label{eqn:intro_manifold_NODE}
    \frac{d}{dt}\Phi_p(t) = \phi_{\Phi_p(t)} \,, \quad \Phi_p(0) = p \,,
\end{equation}
where $p$ is a point in $M$, $\Phi_p:\mathbb{R}\to M$ is the unique solution curve, and $\phi:M\to TM$ is a learnable vector field, e.g., parametrised by a neural network. The NODE can be interpreted as the diffeomorphism $\psi:M\to M$ defined by $\psi:p\mapsto \Phi_p(1)$, mapping the inputs $p\in M$ to the outputs $\Phi_p(1)\in M$ of the model. The situation where the data exhibits global symmetries under some group $G$ requires the construction of equivariant NODEs, which were obtained in~\cite{Kohler2020} for the Euclidean setting and in~\cite{Katsman2021} for Riemannian manifolds $M$. In our previous work~\cite{Andersdotter2025}, we studied equivariant NODEs on manifolds and established conditions under which invariance of the vector field implies equivariance of the resulting flow. In particular, we showed that equivariant manifold NODEs could be further generalised and parametrised in terms of the differential invariants of the action of $G$. We also proved they are universal approximators when $M$ is connected.

Incorporating data that also exhibits non-trivial transformation properties in the presence of local symmetries requires the introduction of feature fields on $M$. The construction of such feature fields has been studied extensively in the context of group-equivariant convolutional neural networks (CNNs) \cite{Cohen2019b,Cheng2019,Gerken2020}. In this setting, features are not treated as scalar-valued functions, but as objects transforming in representations of the local symmetry group $H$. Seminal work by Cohen and collaborators~\cite{Cohen2019b} showed that group-equivariant CNNs can be formulated in terms of induced representations and, more generally, sections of associated vector bundles over homogeneous spaces. From this viewpoint, convolution operators arise as equivariant maps between spaces of sections, or equivalently, intertwiners between representations of $H$, and steerability is a direct consequence of the representation theory inherent in the associated bundle construction. This bundle-theoretic perspective has proven powerful in the discrete-time setting through the connection between convolutions and equivariant feed-forward layers consistent with the global symmetry~\cite{Kondor2018,Aronsson2022}, and more recently in the context of equivariant transformers and their generalisations~\cite{Nyholm2025}.

However, an analogous framework incorporating local symmetry in continuous-time models based on neural ODEs is still missing.
In particular, while NODEs and CNFs describe actions on scalars and scalar densities, there is no general formulation describing flow-based models acting on general vector- or tensor-valued feature fields in a geometrically consistent way that respects the symmetries of the problem.

\subsection{Summary of results}
\label{sec:intro_summary}
The goal of this paper is to fill the aforementioned gap by introducing \emph{steerable neural ODEs}, a geometric extension of NODEs that simultaneously transports points on a manifold and associated feature fields. Our construction is based on the observation that, in the presence of symmetries, feature fields are naturally described as sections of vector bundles rather than as ordinary functions. In particular, when the underlying manifold is a homogeneous space $M = G/H$, feature vectors transform in representations of the stabiliser subgroup $H$, and the natural geometric setting is that of associated vector bundles over $G/H$. Within this framework, the evolution of feature fields along NODE trajectories in $M$ is governed by a connection on the corresponding principal bundle. Equivariance of the NODE on the base manifold under the global group $G$ induces equivariance of the feature transport, provided that the connection is compatible with the group action. This leads to a notion of an induced action on feature fields, respecting all symmetries of the problem and generalising the push-forward of densities in continuous normalizing flows.

More specifically, we consider neural ODEs on the homogeneous space $M=G/H$, where $G$ is a Lie group $G$ and $H$ a closed Lie subgroup. The Lie group $G$ defines a principal bundle $G \to G/H$ with base manifold $G/H$ and fibre $H$. Feature fields are maps $f:G/H \to V$, i.e., sections of the associated vector bundle $G \times_{\rho} V$ with base manifold $G/H$ and fibre $V$, where $V$ is a vector space equipped with a representation $\rho$ of the stabiliser subgroup $H$. A principal connection $\omega \in \Omega^1(G,\mathfrak{h})$ is a differential 1-form on $G$, taking values in the Lie algebra $\mathfrak{h}$ of $H$, which defines a notion of horizontal tangent vectors in the principal bundle. The solution curve of the NODE can be horizontally lifted to the principal bundle and used to define a notion of horizontality inherited by the associated bundle. The action on a feature field $f$ is then obtained as \textit{parallel transport} along $\Phi_p(t)$ by requiring that the corresponding section of $G \times_{\rho} V$ over $\Phi_p(t)$ is everywhere horizontal. The steerable NODE is defined by imposing horizontality to \textit{steer} the feature vector as it is transported along $\Phi_p(t)$. This amounts to extending $\psi$ to the map $\psi : p \mapsto\Phi_p(1)$ to the map $\Psi : M \times V \to M \times V$ defined by $(p,f(p)) \mapsto (\Phi_p(1),f(\Phi_p(1))$ where $f(\Phi_p(1))$ is obtained through parallel transport. The geometric construction of the steerable NODEs is illustrated in~\Cref{fig:steerable_NODE_conceptual}

\begin{figure}[t]
\centering
\begin{subfigure}{0.48\textwidth}
  \includegraphics[height=4.0cm]{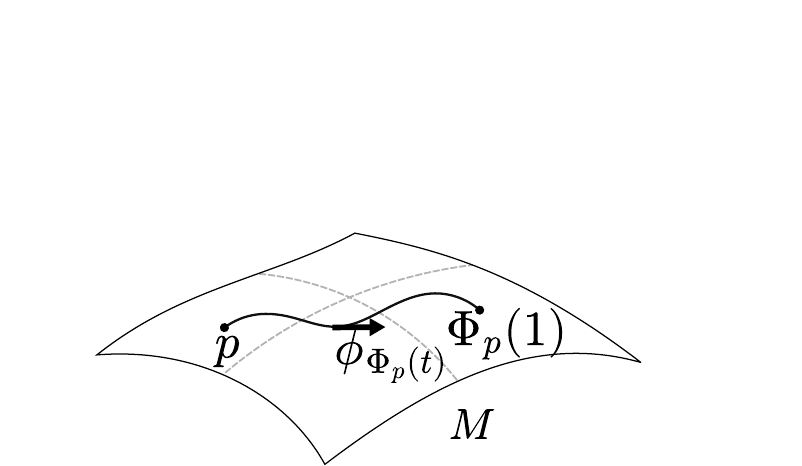}
  \caption{Neural ODE}
  \label{subfig:Classical_NODE}
\end{subfigure}
\hspace{-1cm}
\begin{subfigure}{0.48\textwidth}
  \includegraphics[height=4.0cm]{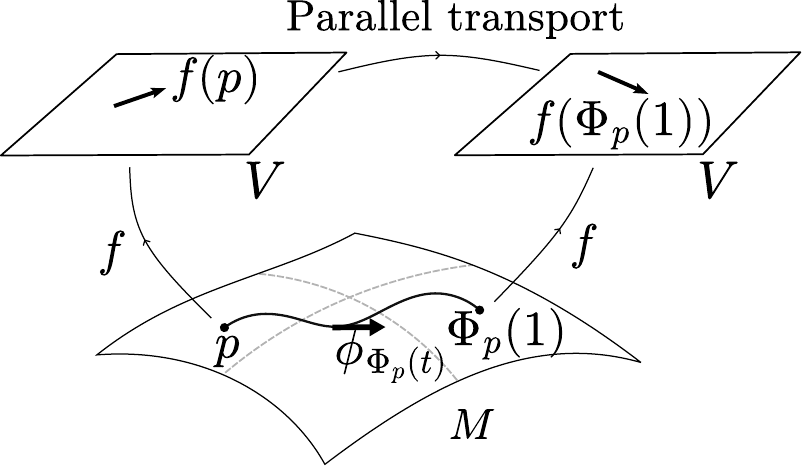}
  \caption{Steerable neural ODE}
  \label{subfig:steerable_NODE}
\end{subfigure}
\caption{In a neural ODE on a homogeneous space $M=G/H$ (\cref{subfig:Classical_NODE}), points $p\in M$ are transported to outputs $\Phi_p(1)$ along the flow governed by a vector field $\phi$. In a steerable NODE (\cref{subfig:steerable_NODE}), this framework is extended to include features that are associated to the points $p$, defined by a feature field $f:M\to V$, where the vector space $V$ carries a representation of $H$. The transformation of the features along the NODE trajectory is governed by parallel transport with respect to a principal connection on $G$. A more detailed illustration of a steerable NODE can be seen in~\cref{fig:Steerable_NODE_illustration}.}
\label{fig:steerable_NODE_conceptual}
\end{figure}

In the steerable NODEs, we treat both the vector field $\phi$ and the connection $\omega$ as learnable quantities. Training the model amounts to learning both the flow in $M = G/H$ and the bundle geometry required to induce the observed transformation of points in $M$ and feature vectors in $V$ attached to those points. In contrast to existing discrete-time models, which probe representation-theoretic aspects of the associated vector bundles, the continuous-time steerable NODEs also probe the geometric structure of the bundles $G$ and $G \times_{\rho} V$ via the connection $\omega$. Consequently, training the steerable NODE not only results in a diffeomorphism approximating the observed transformations of feature vectors, but also provides insight and information regarding the geometry of the underlying problem. In this way, our construction introduces a novel perspective on NODEs in geometric deep learning, incorporating learning of intrinsically geometric structures and broadening the scope to arbitrary feature vectors on homogeneous spaces.

The main contributions we describe in this paper are as follows:
\begin{itemize}
    \item We introduce steerable NODEs on homogeneous spaces $M=G/H$ as parallel transport in associated vector bundles $G \times_{\rho} V$. In~\Cref{dfn:steerable_NODE}, we provide a concrete description of the model as a system of non-linear ODEs; the first describes a flow in $M$, and the second describes the parallel transport steering the feature vector as it is transported along the flow (see~\cref{fig:steerable_NODE_conceptual}).
    \item We show that the steerable NODE is equivariant with respect to the global action of $G$ whenever $\phi$ and $\omega$ are both $G$-invariant (\Cref{thm:equivariance_steerable_NODE_flow}). In~\cref{sec:Wangs_theorem}, we also provide concrete parametrisations of the spaces of invariant vector fields and connections using our previous result from~\cite{Andersdotter2025} and Wang's theorem (\Cref{thm:wang}).
    \item In~\cref{sec:existingNODEs}, we show that existing NODEs on Lie groups are encompassed by steerable NODEs in the special case where $H$ is trivial, and how steerable NODEs extend the modelling capabilities of current flow-based machine learning models. In particular, we demonstrate that steerable NODEs encompass certain continuous normalizing flows for lattice gauge theories studied recently. 
\end{itemize}

\subsection{Organisation of the paper}
\label{sec:intro_organisation}
In~\cref{sec:mathematical_preliminaries}, we review the necessary differential geometric background, including fibre bundles, principal connections, parallel transport, and homogeneous spaces. In particular, we recall the construction of associated vector bundles and the description of sections in terms of Mackey functions and induced representations. We then introduce steerable neural ODEs on homogeneous spaces in ~\cref{sec:steerable_NODEs}. We first reinterpret feature fields as sections of associated bundles and then construct the induced evolution of features via horizontal lifts and parallel transport. This leads to a precise definition of steerable NODEs as coupled flows on base manifolds and associated bundles. In~\cref{sec:induced-equivariance} we prove that invariance of the base vector field together with invariance of the connection implies equivariance of the lifted flow and of the induced feature transport. This establishes sufficient geometric conditions for steerable NODEs to respect global symmetries. Furthermore, we apply Wang’s theorem~\cite{Wang1958} to classify $G$-invariant principal connections on $G \to G/H$. Together with previous results on the classification of invariant base vector fields using differential invariants, this provides an explicit parametrisation of all admissible equivariant steerable NODE models on $G/H$. Finally, in~\cref{sec:existingNODEs}, we discuss the relation of our framework to existing NODE models, including continuous normalizing flows and previously studied equivariant NODE constructions.
We conclude in~\cref{sec:conclusions} with directions for future work.

\section{Mathematical preliminaries}\label{sec:mathematical_preliminaries}
The construction of steerable NODEs presented in the following sections relies on concepts from differential geometry, specifically the notions of fibre bundles and homogeneous spaces. In this section, we recall the relevant definitions and properties of these geometric constructs, assuming throughout that the manifolds and maps considered are all smooth. Proofs and detailed derivations are omitted for brevity, but can be readily found in standard texts on differential geometry such as~\cite{Kobayashi1963,Nakahara2003,Lee2012}. The reader who is already familiar with these topics can skip to \cref{sec:steerable_NODEs}.

\subsection{Fibre bundles}\label{sec:fibre_bundles}
Intuitively, a fibre bundle is a base manifold $M$ with copies of a fibre manifold $F$ attached to each point $p\in M$ and glued together in a possibly twisted manner. The canonical examples are provided by the cylinder $S^1 \times [0,1]$ and the Möbius strip, which have the same local structure -- both are locally homeomorphic to $\mathbb{R}^2$ -- but are globally different due to the twist in the fibre $[0,1]$ over the base manifold $S^1$ in the Möbius strip. The definition of a fibre bundle formalises this notion of a space that is locally, but not globally, a product space.

\begin{definition}[Fibre bundle]\label{dfn:fibre_bundle}
    A \textit{bundle of manifolds} is a triple $(E,\pi,M)$, denoted $E\overset{\pi}{\to}M$, where $E$ and $M$ are manifolds referred to as the \textit{total space} and the \textit{base space}, respectively, and $\pi:E\to M$ is a surjective map called the \textit{projection}. The \textit{fibre} $E_p$ at a point $p\in M$ is the pre-image of the projection, $E_p =\pi^{-1}(p)$. A \textit{fibre bundle} with typical (or canonical) fibre $F$ is a bundle $E\overset{\pi}{\to}M$ together with a local trivialization $\{U_i,\phi_i\}$, where $\{U_i\}$ is an open cover of $M$ and $\phi_i:\pi^{-1}(U_i)\to U_i\times F$ is a diffeomorphism such that $\text{proj}_1\circ\phi_i=\pi$ on $\pi^{-1}(U_i)$.
\end{definition}

A section of a fibre bundle is a map $\sigma : M \to E$ assigning to each point $p$ in $M$ a point in the corresponding fibre $E_p$. In general, a fibre bundle admits no global sections due to topological obstructions. However, local sections obtained by restricting the domain of $\sigma$ to the open sets $U_i$ in $M$ always exist.

\begin{definition}[Section of a bundle]\label{dfn:section_bundle}
    A \textit{section} of a fibre bundle $E\overset{\pi}{\to}M$ is a map $\sigma:M\to E$ such that $\pi\circ\sigma(p)=p$ for all $p\in M$.
\end{definition}

We are particularly interested in fibre bundles admitting a Lie group of local symmetries acting on the total space. Throughout, we will use the notation $L_g$ and $R_g$ to refer to the left and right actions, respectively, of a group element $g\in G$. For the situations where the action is on points on a manifold, we will sometimes use juxtaposition for brevity.

\begin{definition}[Principal bundle]\label{dfn:principal_bundle}
    A \textit{principal $H$-bundle} is a fibre bundle $E\overset{\pi}{\to}M$ together with a Lie group $H$ and a free right action of $H$ on $E$, $E\times H\to E$, which is fibre-preserving (meaning that $\pi(uh) = \pi(u)$ for any $u \in E$ and $h \in H$) and transitive on each fibre.
\end{definition}

The fibres of a principal $H$-bundle are the $H$-orbits, all of which are diffeomorphic (but not isomorphic) to the group $H$. Given a principal $H$-bundle $E$ and a representation $\rho$ of the symmetry group $H$, it is possible to define a vector bundle which inherits the topology and local symmetry of $E$.

\begin{definition}[Associated vector bundle]\label{dfn:associated_bundle}
    Let $E\overset{\pi}{\to}M$ be a principal $H$-bundle, let $V$ be a vector space, and let $\rho:H \to \mathrm{GL}(V)$ be a representation of $H$. Let $\sim$ be the equivalence relation on $E\times V$ defined by $(u,v)\sim (uh,\rho(h^{-1})\,v)$ for any $h\in H$, and denote the corresponding equivalence classes by $[u,v]=\{(uh,\rho(h^{-1})v)\,|\,h\in H\}$. The \textit{associated vector bundle} is then defined as the space
    \begin{equation}
        E \times_{\rho} V := (E\times V)/\sim\,,
    \end{equation}
    equipped with the projection operator $\pi_{\rho}: E \times_{\rho} V \to M$ defined by $\pi_{\rho}([u,v]) = \pi(u)$.
\end{definition}

By construction, the projection is independent of the representative of the equivalence class, since $\pi_{\rho}([uh,\rho(h^{-1})v]) = \pi(uh) = \pi(u) = \pi_{\rho}([u,v])$, and the fibre of $E \times_{\rho} V$ is diffeomorphic to $V$. 

A section of the associated vector bundle can be expressed in terms of a section $\sigma$ of $E$ and a vector field $f:M \to V$ on $M$ as $s(p) = \left[\sigma(p),f(p)\right]$. Under a change of the principal section of the form $\sigma'(p) = \sigma(p)h$ for some $h \in H$, the vector $f(p)$ transforms according to the representation $\rho$ of the local symmetry group $H$, $[\sigma'(p),f(p)] = [\sigma(p)h,f(p)] = [\sigma(p),\rho(h)f(p)]$.

Fixing the section $\sigma$ amounts to choosing a reference in $V$. This can be made manifest through the canonical local trivialisation $\varphi$ defined as follows.

\begin{definition}[Canonical local trivialization]\label{dfn:canonical_local_trivialization}
    The \textit{canonical local trivialisation} $\varphi:M\times V\to E\times_\rho V$ relative to the section $\sigma$ of the associated vector bundle is given by
    \begin{equation}
        \varphi(p,v)=[\sigma(p),v],
    \end{equation}
    where $p\in M$ and $v\in V$.
\end{definition}
The canonical local trivialisation $\varphi$ is a (local) diffeomorphism by construction.

\subsection{Connections and parallel transport}\label{sec:connections_parallel_transport}
In order to relate geometric information defined at different points on $M$, we need a construction relating the fibres of a bundle over different points in the base manifold $M$. For principal bundles, this is provided by a connection, which defines a notion of parallel transport of a point in the total space of the bundle along a curve $M$. This notion is then inherited by the associated vector bundles.

In order to define parallel transport rigorously for a principal $H$-bundle $E\overset{\pi}{\rightarrow}M$, we first need to introduce precise notions of vertical and horizontal directions in the total space $E$. The vertical subspace $V_uE \subset T_uE$ is the component  of the tangent space $T_uE$ at $u \in E$ along the fibre $E_p$ over $p = \pi(u)$,
\begin{equation}
    V_uE := \ker \, (\pi_u)_* = \{X \in T_uE\,|\,(\pi_u)_*(X) = 0\}\,.
\end{equation}
The vertical subspace can be explicitly constructed using the right action of $H$. Let $A\in\mathfrak{h}$ be an element of the Lie algebra $\mathfrak{h}$ of $H$ and define a curve $\Gamma_u$ on $E$ through $u$ by
\begin{equation}
    \Gamma_u(t)=u\exp{(tA)} \,.
\end{equation}
The curve is contained in the fibre at $p$, $\Gamma_u \subset E_p$, since $\pi(u\exp{(tA)}) = \pi(u) = p$, and its tangent defines the \textit{fundamental vector field} $A^\#$ generated by $A$,
\begin{equation}
    A^\#(u)=\left.\frac{d}{dt}u\exp{(tA)}\right|_{t=0}=\dot{\Gamma}_u(0) \,,
\end{equation}
where $\dot{\Gamma}$ denotes the derivative with respect to $t$. Since $\pi\circ\Gamma_p$ is constant with respect to $t$, we have
\begin{equation}
    \pi_*(A^\#(u))=\left.\frac{d}{dt}\pi(\Gamma_u(t))\right|_{t=0}=0
\end{equation}
for any $u\in E$. Consequently, $A^\#(u) \in V_uE$ and $A^\#$ is a vertical vector field on $E$. In fact, the fundamental vector fields provide an isomorphism of vector spaces $\mathfrak{h} \cong V_uE$.

The horizontal subspace $H_uE \subset T_uE$ is the complement of the vertical subspace $V_uE$. However, unlike the vertical subspace, it is not uniquely determined by the bundle topology alone but requires additional geometric structure in the form of a connection that describes how vectors in the tangent space $T_uE$ are projected onto the vertical subspace.

\begin{definition}[Principal connection]\label{def:connection_1_form}
    A (Ehresmann) \textit{connection} 1-form $\omega\in\Omega^1(E,\mathfrak{h})$ on the principal $H$-bundle $E$ is a projection onto the vertical component $V_uE$ satisfying the following requirements:
    \begin{itemize}
        \item[(i)] $\omega(A^\#(u))=A$, for all $u\in E$ and $A\in\mathfrak{h}$. 
        \item[(ii)] $R_h^*\omega=Ad_{h^{-1}}\omega$, for all $h\in H$.
    \end{itemize}
\end{definition}

Given a connection 1-form $\omega$, the horizontal subspace $H_uE$, for any $u\in E$, is defined by
\begin{equation}
    H_uE:= \ker \omega_u = \{X\in T_uE\,|\,\omega_u(X)=0\}\,.
\end{equation}
It follows that the connection defines a separation of the tangent space $T_uE=V_uE\oplus H_uE$ such that any smooth vector field $X$ on $E$ is decomposed into smooth components $X=X^H+X^V$, where $X^H_u\in H_uE$ and $X^V_u\in V_uE$. Moreover, the separation is equivariant with respect to the right action of $H$, $H_{uh}E = R_{h*}H_uE$ for all $u\in E$ and $h\in H$.

The connection $\omega$ provides the desired notion of transporting geometric information along a curve in the base manifold $M$. Starting with the principal bundle, a curve $\gamma$ in $M$ can be lifted to $E$ in a way that respects the decomposition of $T_uE$ induced by $\omega$ by requiring that its tangent vector remains in $H_uE$ at each point.

\begin{definition}[Horizontal lift]\label{dfn:horizontal_lift}
    A \textit{horizontal lift} of a curve $\gamma:I\to M$ is a curve $\tilde{\gamma}:I\to E$ such that $\pi\circ\tilde{\gamma}=\gamma$ and $\dot{\tilde{\gamma}}(t)\in H_{\tilde{\gamma}(t)}E$ for all $t\in I$.
\end{definition}

The horizontal lift can be expressed in terms of a local section $\sigma$ as $\tilde{\gamma}(t) = \sigma(\gamma(t))h(t)$, where $h: \mathbb{R} \to H$. In terms of the connection, the condition defining the horizontal lift $\tilde{\gamma}$ is $\omega(\dot{\tilde{\gamma}}(t))=0$, or equivalently,
\begin{equation}\label{eqn:horizontal_lift_ode}
    \frac{d}{dt}h(t)=-(R_{h})_*\left[ \omega\left(\sigma_{*}\left(\frac{d}{dt}\gamma(t)\right)\right) \right]\,.
\end{equation}
Note that the induced right action $(R_{h})_*$ is acting on the vector in $\mathfrak{h}$ obtained by evaluating $\omega(\sigma_{*}\dot{\gamma}(t))$. This is an ODE for $h(t)$, and standard ODE theory guarantees the existence and uniqueness of a horizontal lift given an initial condition $\tilde{\gamma}(0) \in \pi^{-1}(\gamma(0))$. Since the horizontal subspaces are equivariant, $(R_h)_*H_uE=H_{uh}E$, $H$ acts transitively on the right on the horizontal lifts of $\gamma$ by shifting the initial condition. The point $\tilde{\gamma}(t)$ is called the \textit{parallel transport} of $\tilde{\gamma}(0)$ along $\gamma$. 

The horizontal lift of curves can be extended to a flow $\Phi:\R \times M \to M$ generated by a vector field $\phi: M \to TM$,
\begin{equation}
    \frac{d}{dt} \Phi(t,p) = \phi_{\Phi(t,p)} \,, \quad \Phi(0,p) = p \,.
\end{equation}
Given a fixed connection $\omega$, there is a unique lift of $\phi$ to a vector field $\tilde{\phi}: E \to TE$ which is horizontal, $\omega(\tilde{\phi})=0$, $H$-invariant, $(R_h)_*\tilde{\phi}=\tilde{\phi}$ for all $h \in H$, and satisfies the defining property $\pi_*(\tilde{\phi}) = \phi$ of a lift~\cite{Bleecker1981}. The flow $\tilde{\Phi}:\R \times E \to E$ generated by $\tilde{\phi}$ gives the horizontal lifts of the integral curves of $\phi$ and is equivariant under $H$ by construction, $\tilde{\Phi}(t,u)h = \tilde{\Phi}(t,uh)$ for all $h \in H$. 

Parallel transport in the principal bundle induces a notion of parallel transport in the associated vector bundle in the following way. A section of $E \times_{\rho} V$ along a curve $\gamma:\mathbb{R}\to M$ can always be expressed in terms of the horizontal lift as
\begin{equation}
    s(\gamma(t))=[(\tilde{\gamma}(t),\eta(\gamma(t))],
\end{equation}
where $\tilde{\gamma}$ is the horizontal lift of $\gamma$ and $\eta(\gamma(t))\in V$ is the fibre component. The \textit{covariant derivative} of the section $s$ along $\gamma$ describes how the value of the section changes relative to parallel transport defined by the connection.

\begin{definition}\label{dfn:parallel_transport_associated}
    Let $\phi:M \to TM$ be a vector field and $\gamma:\R \to M$ a curve with $\dot{\gamma}(t)=\phi$. Let $s(\gamma(t)) = [\tilde{\gamma}(t),\eta(\gamma(t))]$ be a section of the associated bundle $E \times_{\rho} V$. Then the \textit{covariant derivative} of $s$ along $\gamma(t)$ is given by
    \begin{equation}\label{eqn:covariant_derivative}
        (\nabla_{\phi}s)(\gamma(t)):=\left[\left(\tilde{\gamma}(t),\frac{d}{dt}\eta(\gamma(t))\right)\right].
    \end{equation}
\end{definition}

We note that the definition of the covariant derivative is independent of the section $\sigma$ used to express the horizontal lift of $\gamma(t)$. Furthermore, since all possible horizontal lifts of $\gamma(t)$ are related by the transitive right action of $H$, the covariant derivative in~\eqref{eqn:covariant_derivative} is, in fact, also independent of the choice of lift or equivalently, the initial condition $\tilde{\gamma}(0)$. The section $s$ undergoes \textit{parallel transport} along $\gamma$ if $(\nabla_{\phi}s)(\gamma(t))=\left[\left(\tilde{\gamma}(t),0\right)\right]$ for all $t$. 

\subsection{Homogeneous spaces}
In the presence of a global Lie group $G$ of symmetries acting transitively on a manifold $M$ on the left, $G \times M \to M$, the manifold can be described as the \textit{homogeneous space} $M = G/H$, where $H$ is the stabiliser Lie subgroup of the action of $G$. Such spaces carry a natural bundle structure, allowing for the machinery of the previous principal and associated bundles over $M$ to be applied.

\begin{example}\label{ex:G_as_principal_bundle}
    Consider the manifold $M = G/H$, where $G$ is a Lie group and $H$ a closed Lie subgroup. Let $\pi:G \to G/H$ be the map defined by $\pi(g)=gH$ for each $g\in G$, that is, the map sending each element $g\in G$ to its corresponding coset. Then $G\overset{\pi}{\to}G/H$ is a fibre bundle with typical fibre $H$; in fact, it is an example of a principal $H$-bundle since $\pi(gh) = gHh = gH = \pi(g)$ and $H$ acts transitively on $\pi^{-1}(gH)$.
\end{example}

The groups $G$ and $H$ serve different geometrical roles. The local right action of $H$ preserves the fibre, while the global left action of $G$ acts transitively, meaning it moves points from one fibre to another. The sections of vector bundles associated to $G\overset{\pi}{\to}G/H$ simultaneously accommodate both these actions in a consistent way.

\begin{example}\label{dfn:associated_bundle_homogeneous}
    Let $V$ be a vector space and $\rho:H \to \mathrm{GL}(V)$ a representation of $H$. Let $G \times_{\rho} V$ be the associated vector bundle with elements $[g,v]=\{(gh,\rho(h^{-1})v)\,|\,h\in H\}$, $(g,v)\in G \times V$, and bundle projection $\pi_{\rho}([g,v])=\pi(g)=gH$. The base space of $G \times_{\rho} V$ is $G/H$ and its typical fibre is $V$. Given a section $\sigma$ of $G$, a section $s$ of $G \times_{\rho} V$ can be expressed as $s(p) = [\sigma(p),f(p)]$, where $f: M \to V$.
\end{example}

The fact that the value of the section $s: M \to V$ is measured relative to a section $\sigma$ of the principal bundle can be made manifest using an alternative perspective available for the case of principal bundles over homogeneous spaces.

\begin{definition}[Mackey function]\label{dfn:mackey_fnc}
    Let $G$ be a Lie group with a closed Lie subgroup $H$. Let $V$ be a vector space $V$ and let $\rho:H\to \mathrm{GL}(V)$ be a representation of $H$. A \textit{Mackey function} is a map $k:G\to V$ satisfying $k(gh)=\rho(h^{-1})k(g)$ for all $h\in H$ and $g\in G$.
\end{definition}

The space $\mathcal{I}$ of Mackey functions is a vector space isomorphic to the space of sections of $G\times_{\rho}V$~\cite{Kolar2013}. Concretely, a section of $G \times_{\rho} V$ can be written as $[\sigma(p),f(p)]$, where $f = k \circ \sigma$ and $k:G \to V$ is a Mackey function. The action of $H$ on $V$ through the representation $\rho$ can be extended to an action of $G$ through the \textit{induced representation} $\mathrm{Ind}^G_H \rho : G \to \mathrm{End}(\mathcal{I})$ defined by $\mathrm{Ind}^G_H \rho (g) k(g') = k(g^{-1} g')$, for $g,g'\in G$. This provides a well-defined notion of a global action of $G$ on the vector-valued sections of $G \times_{\rho} V$.

\section{Steerable neural ODEs on homogeneous spaces}\label{sec:steerable_NODEs}
In this section, we introduce \textit{steerable} neural ODEs on homogeneous spaces $M=G/H$ as a geometric extension of manifold neural ODEs (NODEs) incorporating general feature fields using parallel transport. Our construction draws on the concepts of differential geometry introduced in \cref{sec:mathematical_preliminaries}, in particular horizontal lifts and their properties. The resulting model represents a novel approach to flow-based machine learning models that accommodates features with non-trivial transformation properties under the local action of the stabiliser subgroup $H$.

\subsection{Motivation and general setup}\label{sec:motivation_and_general_setup}
We recall the definition of manifold NODEs~\cite{Falorsi2020,Lou2020,Mathieu2020}, i.e., NODEs that transform points in a manifold $M$ through the diffeomorphism defined by a flow on $M$.

\begin{definition}[Manifold NODE]\label{dfn:mfd_NODE}
    Let $\phi:M\to TM$ be a learnable vector field on $M$. Then for every point $p\in M$, there exists a unique curve $\Phi_p:\mathbb{R}\to M$ that solves the Cauchy problem
    \begin{equation}\label{eqn:mfd_NODE}
        \frac{d}{dt}{\Phi}_p(t) = \phi_{\Phi_p(t)} \,, \quad \Phi_p(0) = p \,.
    \end{equation}
    The \textit{manifold NODE} on $M$ is the diffeomorphism $\psi:M\to M$ defined by $\psi(p) = \Phi_p(1)$. The point $p\in M$ is the \textit{input} and $\psi(p)\in M$ is the \textit{output} of the model. 
\end{definition}

A NODE can, therefore, be thought of as a curve on a manifold $M$ whose tangent vector at each point is determined by a machine learning model. Geometrically, it is convenient to identify the manifold NODE with the flow $\Phi: \R \times M \to M$ obtained by collecting all integral curves $\Phi_p(t)$ of $\phi$, i.e., $\Phi(t,p)=\Phi_p(t)$.

Our objective is to generalise manifold NODEs so that each point on the manifold $M$ carries a feature valued in a vector space $V$, and to describe how these features change as the points move with the flow on $M$ generated by $\phi$. In the presence of a group $G$ of global symmetries acting on $M$, we also require that the features transform in a consistent way. Previous work on symmetries of manifold NODEs (e.g.,~\cite{Kohler2020,Katsman2021,Andersdotter2025}) considered the global action of $G$ that translates points in $M$. When features are attached to points of $M$, these features may also transform under an internal symmetry group. For example, a tangent vector at $p\in M$ can be rotated while remaining in the tangent space $T_pM$.

We assume that $G$ is a Lie group acting smoothly on $M$. To ensure that the action does not distinguish between points in $M$, we also assume that it is transitive, making $M$ into a homogeneous $G$-space, $M=G/H$, with $H$ a closed subgroup of $G$ (see~\cref{sec:mathematical_preliminaries}). This setting naturally provides a notion of local transformations of the feature vectors, through a representation $\rho: H \to \mathrm{GL}(V)$ of the stabiliser subgroup $H$, which is consistent with the global action of $G$.

In order to define an action on feature fields induced by the flow in $M=G/H$, we proceed by leveraging the geometric perspective on feature fields on $M$ as sections of associated bundles.

\subsection{Feature fields and induced representations}\label{sec:feature_fields_and_induced_representations}
We first give a rigorous geometric definition of a feature field, i.e., a smooth map $f:M \to V$ assigning to each point $p$ in the homogeneous space $M=G/H$ an element in some vector space $V$. Since such fields must be compatible with the global action of $G$ on $M$ and with the local action of the stabiliser subgroup $H$, we seek a vector bundle over $M$ that encodes these symmetries. The bundle with these properties is precisely the associated vector bundle $G \times_{\rho} V$ defined in~\cref{dfn:associated_bundle_homogeneous}, which inherits the global action of $G$ from the principal bundle $G \overset{\pi}{\to} G/H$ and a local action of $H$ through the representation $\rho: H \to \mathrm{GL}(V)$. Under this construction, feature fields are naturally identified with smooth sections of the associated bundle.

\begin{definition}[Feature fields]\label{dfn:feature_fields}
    Let $M=G/H$ be a homogeneous space and $V$ be a vector space carrying a representation $\rho:H \to \mathrm{GL}(V)$ of $H$. A \textit{feature field} on $M$ is a section of the associated vector bundle $G \times_{\rho} V$. 
\end{definition}

In applications, it is often convenient to work with a local description of a feature field. We therefore introduce an equivalent description in terms of local sections of the principal bundle and Mackey functions.
\begin{definition}[Local feature fields]\label{dfn:local_feature_fields}
    Let $M=G/H$ be a homogeneous space and let $\sigma:M\to G$ be a local section of the principal bundle $G\overset{\pi}{\to}M$. A \textit{local feature field} relative to $\sigma$ is a map $f:M\to V$ defined by $f=k\circ\sigma$, where $k:G\to V$ is a Mackey function defined in~\cref{dfn:mackey_fnc}.
\end{definition}

\begin{restatable}{lemma}{MackeyLemma}\label{lem:feature_fields}
    \Cref{dfn:feature_fields,dfn:local_feature_fields} are equivalent, establishing a one‑to‑one correspondence between Mackey functions and sections of the associated vector bundle.
\end{restatable}

\begin{proof}
    See~\cref{app:proof_of_equiv_feat_fields}.
\end{proof}

It follows from~\Cref{lem:feature_fields} that, given a local section $\sigma:M \to G$ of the principal bundle, any section $s:M \to G\times_\rho V$ (i.e., feature field) of the associated bundle admits the local expression
\begin{equation}
    s(p) = \left[\sigma(p),f(p)\right] = \left[\sigma(p),k(\sigma(p))\right]\,,
\end{equation}
where $p\in M$, $k:G\to V$ is a uniquely determined Mackey function and $f=k\circ\sigma$ is the local feature field in~\cref{dfn:local_feature_fields}. The section $s$ is independent of the choice of $\sigma$ -- a property inherited from $k$.

While the point $[\sigma(p),k(\sigma(p))]$ in the total space of the bundle $G \times_{\rho} V \to M$ is independent of $\sigma$, the vector in the fibre $V$ corresponding to this point is not. Indeed, the vector $f(p)$ is measured relative to the section $\sigma$, and fixing $\sigma$ amounts to choosing a reference in $V$. This can be done through the canonical local trivialization introduced in \cref{dfn:canonical_local_trivialization}, which is the map $\varphi:M\times V\to G\times_\rho V$ defined by
\begin{equation}
        \varphi(p,v)=[\sigma(p),v],
\end{equation}
where $p\in M$ and $v\in V$. Since the map $\varphi$ is a (local) diffeomorphism, the section $s$ can also be uniquely identified with the map $\varphi^{-1}\circ s:M\to M\times V$ given by
\begin{equation}
    (\varphi^{-1}\circ s)(p)=(p,f(p)).
\end{equation}

Before proceeding, we need to describe the action of the global symmetry group $G$ on sections of the associated bundle $G \times_{\rho} V$. The left action of $G$ on $G \times_{\rho} V$ is required to preserve linearity of the vector part, which forces $L_g[g',v] = [gg',v]$ for all $g,g' \in G$. Since the projection $\pi: G \to G/H$ commutes with the left action, and $H$ acts transitively on the fibres, there is a unique element\footnote{The map $c : G \times M \to H$ is called the \textit{cocycle} associated with $\sigma$. See~\cite{Cohen2019b} for a more detailed discussion.} $c(g,p) \in H$ for every $g \in G$ and $p \in M$ satisfying
\begin{equation}\label{eq:cocycle}
    L_g\sigma(p) = \sigma(gp)c(g,p)\,.
\end{equation}
The following lemma shows that the action of $G$ on feature fields is given by the induced representation on Mackey functions.

\begin{lemma}
    The action of $G$ on $G \times_{\rho} V$ induces an action on sections $s$ of the form $s(p)=[\sigma(p),k(\sigma(p))]$, which is given by the induced representation acting on the Mackey function $k$.
\end{lemma}
\begin{proof}
    From the action on $G \times_{\rho} V$, the definition \eqref{eq:cocycle}, and the equivalence relations of the associated bundle, we get
    \begin{equation}
        L_gs(p) = [L_g\sigma(p),k(\sigma(p))] = [\sigma(gp),\rho(c(g,p))k(\sigma(p))].
    \end{equation}
    Using the fact that $k(\sigma(p))=k(L_{g^{-1}}L_g\sigma(p))=k(L_{g^{-1}}\sigma(gp)c(g,p))$ and the property of Mackey functions, we see that
    \begin{equation}
        L_gs(p) = [\sigma(gp),k(L_{g^{-1}}\sigma(gp))].
    \end{equation}
    Here, $k(L_{g^{-1}}\sigma(gp))$ is exactly the induced representation $\mathrm{Ind}^G_H\rho$ of $g$ on $k(g)$ defined below~\cref{dfn:mackey_fnc}. Thus,
    \begin{equation}
        L_gs(p) = [\sigma(gp),\mathrm{Ind}^G_H\rho(g)k(\sigma(gp))].
    \end{equation}
\end{proof}

Finally, we want to define the left action of $G$ on $M\times V$ in a way that is consistent with the left action on $G\times_\rho V$. Through the canonical local trivialisation $\varphi$, this can be obtained as 
\begin{equation}\label{eq:left_action_on_M_times_V}
    L_g(p, v)
:= \varphi^{-1}\big([L_g \sigma(p), v]\big)
= (gp, \rho(c(g, p)) v)
\end{equation}
for each $(p,v)\in M\times V$.

We have now provided a rigorous geometric definition of a feature field, both from the perspective of a section of an associated bundle and from the equivalent local perspective in terms of Mackey functions. In addition, we have derived how each feature field representative transforms under the left group action of $G$. The results are summarised in~\Cref{tab:feature_field_transformation}.

\begin{table}[!t]
    \centering
    \begin{tabular}{|c|c|}
    \hline
    \thead{\textbf{Feature field representative}} & \thead{\textbf{Transformation under the left action of $g\in G$}}  \\
    \hline
    \makecell{Section \\ $s:M\to G\times_\rho V$} & $[\sigma(p),f(p)]\mapsto[L_g\sigma(p),f(p)]$  \\ 
    \hline
    \makecell{Local feature field \\$f:M\to V$}& $f(p)\mapsto\rho(c(g,p))f(p)$  \\ 
    \hline
    \makecell{Mackey function \\$k:G\to V$}& $k(\sigma(p))\mapsto\mathrm{Ind}^G_H\rho(g)k(\sigma(gp))$  \\ 
    \hline
    \makecell{Local section \\$\varphi^{-1}\circ s:M\to M\times V$} & $(p,f(p))\mapsto(gp,\rho(c(g,p))f(p))$\\
    \hline
    \end{tabular}
    \caption{Summary of the four equivalent descriptions of a feature field introduced in~\cref{sec:feature_fields_and_induced_representations}, together with their transformation behaviours under the left action of $G$.}
    \label{tab:feature_field_transformation}
\end{table}

\subsection{Steerable NODEs through parallel transport}\label{sec:Parallel_transport_in_an_NODE}
We now have the necessary constructions to accomplish the goal set forth in~\cref{sec:motivation_and_general_setup}: to define a NODE model in which feature fields on $M=G/H$ are transported along a flow in a way that respects the geometry and symmetries of the underlying space. We refer to this novel class of models as \textit{steerable} NODEs, since the feature vectors are not only translated using the global action of $G$, but simultaneously \textit{steered} using the local action of $H$. As we will see, this local action is determined by the transport flow and bundle geometry alone, but is independent of the value of the feature field. This motivates the nomenclature in analogy with the original notion of steerability in CNNs~\cite{Cohen2017}.

The starting point is the manifold NODE in~\cref{dfn:mfd_NODE} determined by the vector field $\phi$ or, equivalently, the flow $\Phi$ generated by $\phi$. A principal $H$-connection $\omega$ on $G \to G/H$ provides a consistent way to transport points in the principal bundle along this flow using parallel transport. Feature fields are sections of the associated bundle $G \times_{\rho} V$, which inherits the notion of parallel transport from $G \to G/H$. The parallel transport in the associated bundle will define our steerable NODEs.

Concretely, to describe parallel transport along the integral curves $\Phi_p$ of $\phi$ in~\eqref{eqn:mfd_NODE}, we use the ability to express a section of $G \times_{\rho} V$ over $\Phi_p$ using the horizontal lift $\tilde{\Phi}_p$ to the principal $H$-bundle $G$, defined by
\begin{equation}
    \pi_*\left(\frac{d}{dt}\tilde{\Phi}_p(t)\right) = \frac{d}{dt}\Phi_p(t)\,, \quad \omega\left(\frac{d}{dt}\tilde{\Phi}_p(t)\right) = 0\,.
\end{equation}
In terms of a section $\sigma$ of the principal bundle $G \to G/H$, we have $\tilde{\Phi}_p(t) = \sigma(\Phi_p(t))h_p(t)$, where $h_p(t)$ is determined by~\eqref{eqn:horizontal_lift_ode}, and a section over $\Phi_p$ is
\begin{equation}\label{eq:section_associated_bundle}
    s(\Phi_p(t)) = \left[\sigma(\Phi_p(t)), f(\Phi_p(t))\right] = \left[\tilde{\Phi}_p(t), \rho(h_p^{-1}(t))f(\Phi_p(t))\right]\,.
    \end{equation}
According to~\cref{sec:connections_parallel_transport}, the section undergoes parallel transport along $\Phi_p$ if 
\begin{equation}\label{eqn:parallel_transport_section}
   \nabla_\phi s(\Phi_p(t)) =\left[\tilde{\Phi}_p(t), \frac{d}{dt}\big(\rho(h_p^{-1}(t))f(\Phi_p(t))\big)\right]=\left[\tilde{\Phi}_p(t), 0\right]\,,
\end{equation}
or, equivalently, if $\rho(h_p^{-1}(t))f(\Phi_p(t))$ is constant.

Just as parallel transport in the principal bundle $G$ can be described as the flow $\tilde{\Phi}$ obtained through the horizontal lift as described in~\cref{sec:connections_parallel_transport}, so too can the parallel transport in the associated bundle be described by a horizontal flow.

\begin{definition}[Parallel transport flow]\label{dfn:flow_ass_bundle}
    Let $M=G/H$ be a homogeneous space, $\phi$ a vector field on $M$, $\omega$ a principal $H$-connection on $G\to G/H$, and $G\times_{\rho} V$ the associated bundle defined by the representation $\rho: H \to \mathrm{GL}(V)$. Let $\Phi$ be the flow in $M$ generated by $\phi$, and $\tilde{\Phi}$ be the unique horizontal lift of $\Phi$ with respect to $\omega$. We define the map $\Gamma_{\Phi}:\R \times (G \times_{\rho} V) \to G \times_{\rho} V$ by
    \begin{equation}
        \Gamma_{\Phi}(t,[g,v]) = [\tilde{\Phi}(t,g),v]
    \end{equation}
    for every $g \in G$ and $v \in V$.
\end{definition}

\begin{lemma}\label{lem:flow_ass_bundle}
    The map $\Gamma_{\Phi}$ in~\cref{dfn:flow_ass_bundle} is a horizontal flow in $G$ with respect to the connection induced from $\omega$, and describes the parallel transport in $G \times_{\rho} V$ along the flow $\Phi$ in $M$.    
\end{lemma}
\begin{proof}
    We begin by showing that $\Gamma_\Phi$ is well-defined. We know that $(t,[g,v])=(t,[gh,\rho(h^{-1})v])$ for all $h\in H$. From the defining properties of $\omega$, we also know that $\tilde{\Phi}(t,gh)=\tilde{\Phi}(t,g)h$ for all $h\in H$. Thus, we have
    \begin{equation}
        \Gamma_\Phi(t,[gh,\rho(h^{-1})v]) = [\tilde{\Phi}(t,gh),\rho(h^{-1})v] = [\tilde{\Phi}(t,g),v] = \Gamma_\Phi(t,[g,v])\,,
    \end{equation}
    and $\Gamma_\Phi$ is well-defined.

    Since $\tilde{\Phi}$ is a flow, it immediately follows that $\Gamma_\Phi(0,[g,v])=[g,v]$ and that $\Gamma_\Phi(s,\Gamma_\Phi(t,[g,v]))=\Gamma_\Phi(s+t,[g,v])$ for all $g\in G$, $v\in V$ and $s,t\in\mathbb{R}$, making $\Gamma_\Phi$ a flow on the associated bundle.

    Given $g\in G$, let $p=\pi(g)$ and set $\Phi_p(t)=\Phi(t,p)$. Because $\tilde{\Phi}$ is the horizontal lift of $\Phi$, the horizontal lift of $\Phi_p(t)$ through $g$ is $\tilde{\Phi}_p(t)=\tilde{\Phi}(t,g)$, and
    \begin{equation}
        \Gamma_\Phi(t,[g,v])=[\tilde{\Phi}_p(t),v].
    \end{equation}
    Comparing with~\cref{eq:section_associated_bundle}, this is exactly the form of a parallel transported section where the element $v$ corresponds to the constant quantity $\rho(h_p^{-1}(t))f(\Phi_p(t))$. It follows that $\Gamma_\Phi$ is a flow that transports the element $[g,v]$ along $\Phi_p$ by parallel transport. 

    Moreover, the map $t\mapsto [\tilde{\Phi}_p(t),v]$ defines a curve in $G\times_\rho V$ whose $G$-component is horizontal with respect to $\omega$ and whose $V$-component remains constant. Taking this parallel transport to define an induced connection on the associated bundle, the curve is manifestly horizontal. In this sense, $\Gamma_\Phi$ is a horizontal flow.
\end{proof} 

The flow $\Gamma_\Phi$ provides a geometrically principled way of transporting feature fields (i.e., sections of $G \times_{\rho} V$) along the flow corresponding to the manifold NODE~\eqref{eqn:mfd_NODE}. If $\Phi_p$ is a solution curve of a NODE defined by the flow $\Phi$ generated by $\phi$, and $h_p(t)\in H$ is the element in $H$ with initial condition $h_p(0)=e$ that defines the horizontal lift $\tilde{\Phi}_p(t)=\sigma(\Phi_p(t))h_p(t)$, then
\begin{equation}
    \Gamma_\Phi(1,[\sigma(p),v])=[\tilde{\Phi}_p(1),v]=[\sigma(\Phi_p(1)),\rho(h_p(1))v]
\end{equation}
is the result of parallel transport of the section $s(p)=[\sigma(p),v]$ along the trajectory $\Phi_p$ from $p$ to $\Phi_p(1)$. Furthermore, since $\Gamma_\Phi$ is a flow and the local trivialisation $\varphi$ is a (local) diffeomorphism, the map
\begin{equation}
    (p,v)\overset{\varphi}{\longmapsto}[\sigma(p),v]\mapsto\Gamma_\Phi(1,[\sigma(p),v])\overset{\varphi^{-1}}{\longmapsto}(\Phi_p(1),\rho(h_p(1))v)
\end{equation}
is a diffeomorphism that maps the initial data $(p,v)\in M\times V$ to the output $(\Phi_p(1),\rho(h_p(1))v)$ by transporting the base point $p$ along the flow $\Phi$ and simultaneously steering the feature vector $v$ through the action of $\rho(h_p(1))$. This diffeomorphism defines the steerable NODE.

\begin{definition}[Steerable NODEs on homogeneous spaces]\label{dfn:steerable_NODE}
    Let $M=G/H$ be a homogeneous space, $\rho: H \to \mathrm{GL}(V)$ be a representation of $H$, and let $\sigma$ be a local section of the principal $H$-bundle $G$. Furthermore, let $\phi:M\to TM$ be a learnable vector field on $M$ and $\omega \in  \Omega^1(G,\mathfrak{h})$ a learnable principal $H$-connection on $G$. Then for every point $p\in M$, there exist unique curves $\Phi_p:\mathbb{R}\to M$ and $h_p:\mathbb{R}\to H$ that solve the parallel transport ODEs 
    \begin{empheq}[left=\empheqlbrace]{align}
        &\frac{d\Phi_p(t)}{dt} = \phi_{(\Phi_p(t))},\label{eqn:steerable_NODE_base} \\
        &\frac{dh_p}{dt} = - (R_{h_p(t)})_* \omega\left(\sigma_*\phi_{(\Phi_p(t))}\right),\label{eqn:steerable_NODE_connection}
    \end{empheq}
    with initial conditions $\Phi_p(0)=p$ and $h_p(0)=e$. The \textit{steerable NODE} on $M \times V$ is the diffeomorphism $\Psi:M\times V\to M\times V$ defined by $\Psi(p,v)=(\Phi_p(1), \rho(h_p(1))v)$. The point $(p,v)\in M\times V$ is the \textit{input} and $\Psi(p,v)\in M\times V$ is the \textit{output} of the model.
\end{definition}

\begin{remark}
As noted in~\cref{sec:connections_parallel_transport}, the right-hand side of~\eqref{eqn:steerable_NODE_connection} should be interpreted as the induced right action $(R_{h_p(t)})_*$ on the Lie algebra-valued quantity obtained by evaluating the connection on the vector $\sigma_*\phi_{\Phi_p(t)}$ tangent to $G$. The different maps featured in the construction are illustrated in the following diagram:
\[
\begin{tikzcd}[column sep=4.5em, row sep=3em]
    T_{\sigma(\Phi_p(t))}G 
        \arrow[r, "\omega"] 
    & \mathfrak{h}
        \arrow[out=10, in=170, loop right, "\,(R_{h_p(t)})_*\,"] \\
    T_{\Phi_p(t)}M 
        \arrow[u, "\sigma_*"']
\end{tikzcd}
\]
\end{remark}

The pointwise formulation of steerable NODEs in~\cref{dfn:steerable_NODE} shows that we can indeed interpret the model as a coupled system of (neural) ODEs, where the gradients are determined by the vector field $\phi$ and the connection $\omega$, and where we recover the manifold NODE when $H$ is trivial. We emphasise that both $\phi$ and $\omega$ are learnable and can be parametrised, e.g., by neural networks during implementation. Given an input $(p,v)$, measured with respect to the section $\sigma$ of the principal bundle $G$, we first integrate~\eqref{eqn:steerable_NODE_base} to obtain the curve $\Phi_p(t)$ on $M=G/H$ along which to transport the feature vector. Note that this is a manifold NODE on $M$ independent of the connection. Subsequently, we integrate~\eqref{eqn:steerable_NODE_connection}, which depends on both the vector field $\phi$ and the connection $\omega$, to obtain the steering function $h_p(t)$. The model output $\Psi(p,v) = (\Phi_p(1),\rho(h_p(1)v)$ can then be compared to a target, and the parameters defining $\phi$ and $\omega$ updated to minimise the discrepancy. In particular, it is clear that the available frameworks for training NODEs (e.g., the adjoint sensitivity method~\cite{Chen2018} or the flow matching framework~\cite{Lipman2023}) are then directly applicable to steerable NODEs.

The feature fields are sections of $G \times_{\rho} V$, and are therefore inherently local since, in general, this bundle admits no global sections. This is also reflected in the steerable NODE and its dependence on the local section $\sigma$ of $G$. A global extension to all of $M$ requires using an atlas of charts covering $M$ and the transition functions defining the bundles $G$ and $G \times_{\rho} V$. This is standard practice in differential geometry (see, e.g.,~\cite{Nakahara2003,Lee2012}) and we will not go into detail here. 

The steerable NODE can be interpreted geometrically as a model that learns a flow in the associated bundle $G \times_{\rho} V$ that transforms the feature vector $f(p) \in V$ at $p \in M$ by transporting it along a flow in $M$ in a way that is manifestly compatible with the global symmetry $G$ of $M$. The flow in $M$ is determined by the choice of vector field, and the transport can be accomplished in different ways corresponding to the choice of the connection $\omega$. Consequently, training the steerable NODE amounts to learning both where to transport $f(p)$ and in which of the possible ways to perform the transport.

\Cref{fig:Steerable_NODE_illustration} provides a visual illustration of a steerable NODE, showing how the coupled system of (neural) ODEs interact to transport points $p$ on $M=G/H$ along the solution curve $\Phi_p$, while simultaneously steering the associated features $f(p)$ through the induced action of the solution curve $h_p$ on $H$.

\begin{figure}[!t]
\centering
\includegraphics[scale=0.5]{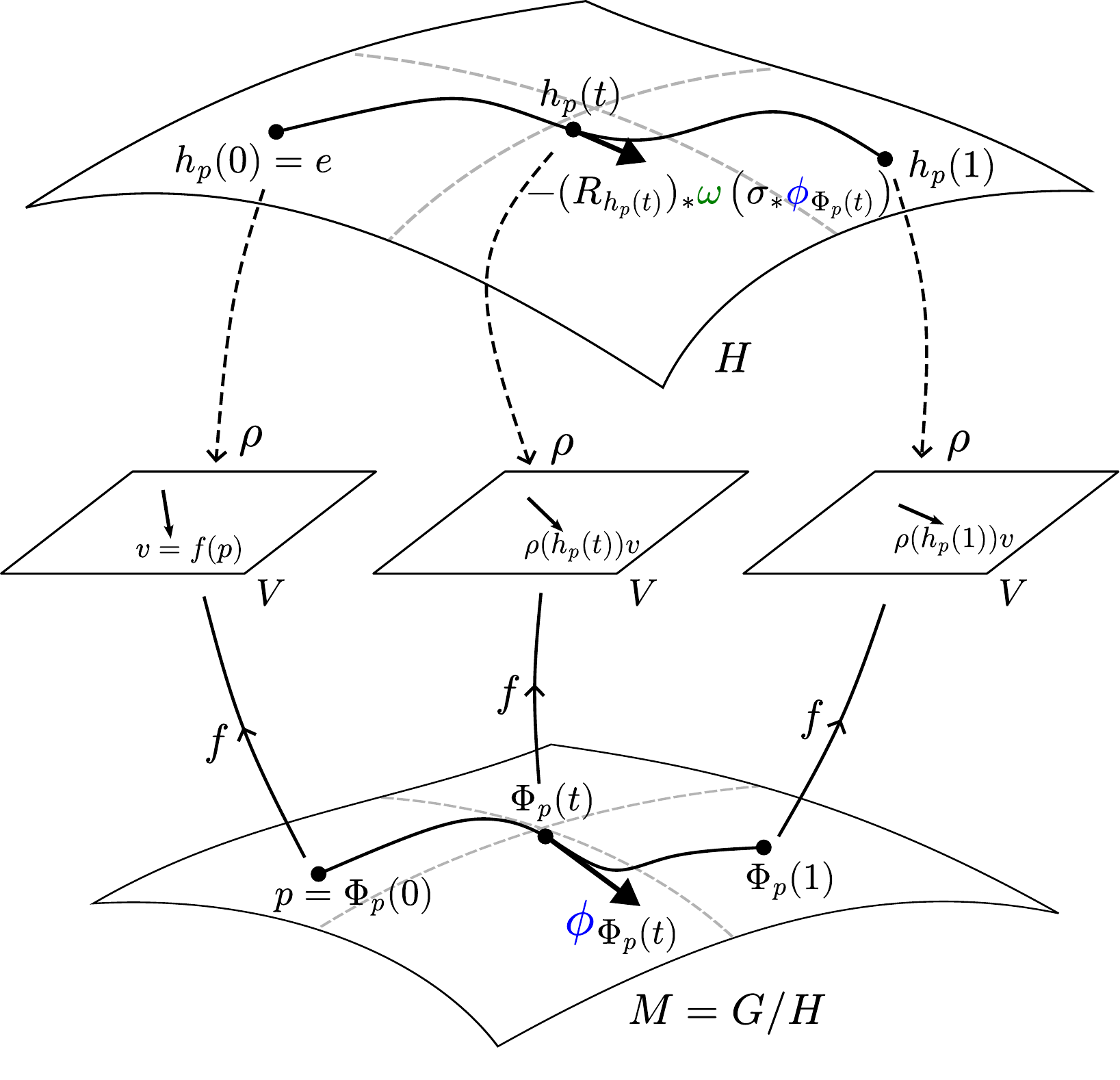}
\caption{A steerable NODE consists of a coupled system of ODEs evolving on different manifolds. The integral curve $\Phi_p$ (bottom) is defined on the homogeneous space $M=G/H$ by a learnable vector field $\phi$ (blue). The curve $h_p$ (top) is defined on the closed Lie subgroup $H$ and is determined by a learnable connection $\omega$ (green). Feature vectors $v\in V$, viewed as elements in the fibre of the associated bundle $G\times_\rho V$, are transported along the curve $\Phi_p$ and simultaneously \textit{steered} by the representation $\rho(h_p(t))$ for each $t\in[0,1]$. The resulting map $(p,v)\mapsto(\Phi_p(1),\rho(h_p(1))v)$ defines the steerable NODE. Notably, the curve $\Phi_p$ depends only on $\phi$, whereas the curve $h_p$, which determines the transformation of the feature component, depends on both $\phi$ and the connection $\omega$. Geometrically, the steering by $\rho(h_p(t))$ describes a parallel transport along the curve $\Phi_p$ in the associated bundle $G\times_\rho V$.}
\label{fig:Steerable_NODE_illustration}
\end{figure}

We end this section by providing a concrete example of a steerable NODE on a topologically trivial homogeneous space with fixed vector field $\phi$ and connection $\omega$.
\begin{example}\label{ex:steerabl_NODE}
    Let $G = \R^2 \times U(1)$ be the trivial $H=U(1)$ bundle over $M=G/H=\R^2$ with the global $G$-action given by translations. A point in $G$ is $(x,y,\theta)$ with $(x,y)\in \R^2$ and $\theta\in \R/2\pi\Z$, and the action of $g=(\alpha,\beta,\varphi) \in G$ is
    \begin{equation}\label{eqn:counterex_left_action_G}
        L_g(x,y,\theta) = (x+\alpha,y+\beta,\theta+\varphi\,(\mathrm{mod}\, 2\pi))\,.
    \end{equation}
    The induced action on $M=\R^2$ is
    \begin{equation}\label{eqn:counterex_left_action_M}
        L_g(x,y) = (x+\alpha,y+\beta)\,,
    \end{equation}
    with stabiliser $H = \{h = (0,0,\varphi) | \varphi \in \R/2\pi\Z\}$ and right action $R_h(x,y,\theta) = (x,y,\theta+\varphi\,(\mathrm{mod}\, 2\pi))$. Note that the adjoint action $\mathrm{Ad}_{h}$ is trivial since $H$ is abelian, and that $(R_h)_* = \id$.

    We consider the vector field $\phi= \partial_x$, with integral curve $\Phi_p(t) = p + (t,0)$ through the point $p \in \R^2$. To define a steerable NODE, we also need a connection on $G$. We take $\omega = \pi dx + dy + d\theta$, which is clearly a 1-form on $G$ taking values in $\mathfrak{h} = \R$ and indeed a principal connection, since:
    \begin{enumerate}
        \item[(i)] It acts trivially on the vertical subspace,
        \begin{equation}
            \omega(\partial_{\theta}) = \pi dx(\partial_{\theta}) + dy(\partial_{\theta}) + d\theta(\partial_{\theta}) = 1 \,.
        \end{equation}
        \item[(ii)] It is $H$-equivariant,
        \begin{equation}
            (R_h)^*\omega = \pi(R_h)^*dx + (R_h)^*dy + (R_h)^*d\theta = \omega =  \mathrm{Ad}_{h^{-1}}\omega\,. 
        \end{equation}
    \end{enumerate}

    We choose a section of $G$ to be $\sigma(x,y)=(x,y,\chi)$, for some constant $\chi \in \R/2\pi\Z$, which yields $\sigma_* \phi = \partial_x$. Inserting this into~\eqref{eqn:steerable_NODE_connection}, we obtain
    \begin{equation}
        \frac{dh_p}{dt} = -\omega(\sigma_* \phi) = -\omega(\partial_x) = -\pi \,,
    \end{equation}
    where we have again used the triviality of the adjoint $H$-action. Solving this equation yields $h_p(t) = -\pi t$, where we have enforced the initial condition $h_p(0)=0$. The horizontal lift is given by
    \begin{equation}
        \tilde{\Phi}_p(t) = \sigma(\Phi_p(t))h(t) = (x+t,y,\chi-\pi t\,(\mathrm{mod}\, 2\pi)) \,,
    \end{equation}
    and the corresponding lift of the vector field is $\tilde{\phi} = \partial_x -\pi \partial_{\theta}$. Indeed, this is horizontal, $\omega(\tilde{\phi}) = \pi - \pi = 0$, and projects to $\phi$, $\pi_* \tilde{\phi} = \partial_x$.

    For the associated vector bundle, we choose a non-trivial representation of $H=U(1)$ on $V=\R^2$ defined by rotations
    \begin{equation}
        \rho(\varphi) = R(\varphi) = \left( \begin{array}{cc}
             \cos(\varphi) & -\sin(\varphi) \\
             \sin(\varphi) & \cos(\varphi) 
        \end{array}\right)\,.
    \end{equation}
    Fixing a pair $(p,v)$ in $M \times V = \R^2 \times \R^2$, the steerable NODE defined by the vector field $\phi$ and the connection $\omega$ is then given by~\cref{dfn:steerable_NODE} as
    \begin{equation}
        \Psi(p,v) = (\Phi_p(1),\rho(h_p(1))v) = (x+1,y,R(\pi)v)\,,
    \end{equation}
    where we write $p = (x,y)$ and have used that $h_p(1) = -\pi = \pi \in \R /2\pi \Z$.
\end{example}

\section{Equivariance of steerable NODEs}\label{sec:induced-equivariance}
In the previous section, we constructed steerable NODEs on the homogeneous space $M=G/H$, where the feature fields transform in a representation $\rho$ of the stabilising subgroup $H$. This guarantees equivariance with respect to the local $H$-action. Beyond this, the global symmetry group $G$ also has a natural action on the feature maps through the induced representation $\mathrm{Ind}_H^G(\rho)$. This raises the question whether steerable NODEs can be made equivariant with respect to the global action of $G$, and what additional constraints such equivariance would impose on the model.

To address these questions, we first need a clear definition of equivariance of steerable NODEs. For ordinary manifold NODEs on $M=G/H$ in~\cref{dfn:mfd_NODE}, which correspond to the case of trivial $H$ in the steerable framework, equivariance is defined as follows.

\begin{definition}\label{dfn:equivariant_mfd_NODE}
    A manifold neural ODE $\psi:M\to M$ is \textit{$G$-equivariant} if $\psi(gp)=L_g\psi(p)$ for all $p\in M$ and $g\in G$.
\end{definition}

It is well known that equivariance of the diffeomorphism $\psi$, equivariance of the corresponding flow, $L_g \Phi(t,p) = \Phi(t,gp)$, and invariance of the generating vector field, $(L_g)_*\phi_p = \phi_{gp}$, are equivalent~\cite{Kohler2020,Katsman2021,Andersdotter2025}.

A notion of equivariance of the steerable NODEs in~\cref{dfn:steerable_NODE} can be obtained in an analogous way, using the left action of $G$ on $M \times V$ introduced in~\cref{sec:feature_fields_and_induced_representations}.

\begin{definition}\label{dfn:equivariant_steerable_NODE}
    A steerable NODE $\Psi:M \times V \to M \times V$ on the homogeneous space $M=G/H$ is \textit{$G$-equivariant} if $L_g\Psi(p,v)=\Psi(L_g(p,v))$ for all $p\in M$, $v \in V$ and $g\in G$.
\end{definition}

\begin{remark}
    The definition of equivariance of the steerable NODE is inherently local due to the local section $\sigma$ defining the coordinates in which the vector $v$ is measured. Extending this globally or changing the gauge $\sigma$ is straightforward using the transition functions defining $G$ and $M=G/H$, in the same way as was discussed in connection with~\cref{dfn:steerable_NODE}. 
\end{remark}

\subsection{Equivariance via horizontal flows}
In this section, we provide sufficient conditions for equivariance of steerable NODEs under the global symmetry group $G$. More specifically, in~\cref{thm:equivariance_steerable_NODE_flow} we show how a steerable NODE inherits the equivariance of the NODE on the base manifold $M=G/H$ if the connection defining parallel transport is compatible with the left action of $G$.

Equivariance of the steerable NODE is conveniently described in terms of the flows associated with the manifold NODE on $M=G/H$. As discussed in~\Cref{sec:connections_parallel_transport}, given a connection $\omega$, the flow $\Phi:\R \times M \to M$ corresponding to the NODE $\psi:M \to M$ lifts uniquely to a horizontal flow $\tilde{\Phi}: \R \times G \to G$ on $G$, which is equivariant with respect to the right action of $H$. In the following \namecref{lem:equiv_lift}, we show that if the flow $\Phi$ is $G$-equivariant, so is the horizontal lift, provided that the connection respects the group structure of $G$.  

\begin{lemma}\label{lem:equiv_lift}
    Let $\Phi$ be a flow on the homogeneous space $M=G/H$ generated by a vector field $\phi: M \to TM$, and let $\omega$ be a fixed principal $H$-connection on $G$. If the vector field $\phi$ and the connection $\omega$ are both $G$-invariant, then the horizontal lift $\tilde{\Phi}$ of $\Phi$ is $G$-equivariant and the corresponding horizontal lift $\tilde{\phi}$ of $\phi$ is $G$-invariant.
\end{lemma}
\begin{proof}
    Since $\tilde{\Phi}$ is the flow generated by $\tilde{\phi}$, it suffices to prove that $(L_g)_*\tilde{\phi}_{g'}=\tilde{\phi}_{gg'}$ for all $g,g'\in G$ if $(L_g)_*\phi_p = \phi_{gp}$ and $(L_g)^*\omega_{g'}=\omega_{g^{-1}g'}$ for all $g,g'\in G$ and $p \in M$ (see, e.g.,~\cite[Theorem 3.4]{Andersdotter2025}). We do this by showing that $(L_g)_*\tilde{\phi}$ is a horizontal $H$-equivariant lift of $\phi$.

    First, we use the fact that the projection $\pi:G\to G/H$ commutes with the left action of $G$, which implies that $\pi_* \circ (L_g)_* = (L_g)_* \circ \pi_*$. For all $g,g'\in G$ we then have that
    \begin{equation}
        \pi_*\left((L_g)_*\tilde{\phi}_{g'}\right) = (L_g)_*\left(\pi_*\tilde{\phi}_{g'}\right) = (L_g)_*\phi_{\pi(g')} = \phi_{L_g\pi(g')} = \phi_{\pi(gg')}\,,
    \end{equation}
    where we have used the fact that $\tilde{\phi}$ is a lift of $\phi$, and that $\phi$ is $G$-invariant. Consequently, $(L_g)_*\tilde{\phi}$ is a lift of $\phi$. Furthermore, the lift $(L_g)_*\tilde{\phi}$ is horizontal by the invariance of $\omega$ since, for all $g,g' \in G$,
    \begin{equation}
        \omega_{gg'}\left((L_g)_*\tilde{\phi}_{g'}\right) = (L_g)^*\omega_{gg'}(\tilde{\phi}_{g'}) = \omega_{g'}(\tilde{\phi}_{g'}) = 0\,.
    \end{equation}
    Finally, equivariance of $(L_g)_*\tilde{\phi}$ under the right action of $H$ is immediate. For all $g,g' \in G$ and $h \in H$ we have
    \begin{equation}
        (R_h)_*\left((L_g)_*\tilde{\phi}_{g'}\right) = (L_g)_*\left((R_h)_*\tilde{\phi}_{g'}\right) = (L_g)_*\tilde{\phi}_{g'h}\,.
    \end{equation}
    
    Consequently, $L_g\tilde{\phi}$ is a horizontal, $H$-equivariant lift of $\phi$. Since $\tilde{\phi}$ is the unique vector field on $G$ with these properties, we must have $L_g\tilde{\phi} = \tilde{\phi}$ for all $g\in G$.
\end{proof}

The parallel transport in the associated bundle $G \times_{\rho} V$ is induced by the parallel transport defined by the horizontal lift $\tilde{\Phi}$. Equivariance under the global action of $G$ is therefore directly inherited by the parallel transport flow $\Gamma_{\Phi}$ defined in~\cref{dfn:flow_ass_bundle}.

\begin{lemma}\label{lem:equiv_lift_ass}
    Let $\Phi$ be the flow on the homogeneous space $M=G/H$ generated by a vector field $\phi: M \to TM$, and let $\omega$ be a principal $H$-connection on $G$. The horizontal lift $\Gamma_{\Phi}: \R \times (G \times_{\rho} V) \to G \times_{\rho} V$ of $\Phi$ in~\cref{dfn:flow_ass_bundle}, defining parallel transport in the associated bundle $G \times_{\rho} V$, is $G$-equivariant if the vector field $\phi$ and the connection $\omega$ are both $G$-invariant.
\end{lemma}
\begin{proof}
    By~\Cref{lem:equiv_lift}, the horizontal lift $\tilde{\Phi}$ of the flow $\Phi$ to the principal bundle $G$ is $G$-equivariant. The left action of $G$ on $G\times_{\rho}V$ is given by $L_g[g',v] = [gg',v]$, which implies
    \begin{eqnarray*}
        L_g \Gamma_{\Phi}(t,[g',v]) &=& L_g\left[\tilde{\Phi}(t,g'),v\right] = \left[L_g\tilde{\Phi}(t,g'),v\right] = \left[\tilde{\Phi}(t,gg'),v\right]\\
        &=& \Gamma_{\Phi}(t,[gg',v]) = \Gamma_{\Phi}(t,L_g[g',v])
    \end{eqnarray*}
    for every $g,g'\in G$ and $v \in V$. Consequently, the flow $\Gamma_{\Phi}$ is $G$-equivariant.
\end{proof}

We are now ready to prove the main theorem establishing $G$-equivariance of the steerable NODE under the conditions that the vector field generating the flow in $G/H$ and the $H$-principal connection defining parallel transport are both $G$-invariant.

\begin{theorem}\label{thm:equivariance_steerable_NODE_flow}
    Let $M=G/H$ be a homogeneous space, let $\phi: M \to TM$ be a vector field on $M$, and let $\omega$ be a principal $H$-connection on $G$. The steerable NODE defined by $\phi$ and $\omega$ is $G$-equivariant if the vector field $\phi$ and the connection $\omega$ are both $G$-invariant.
\end{theorem}
\begin{proof}
    The steerable NODE $\Psi: M \times V \to M \times V$ can be expressed in terms of the flow $\Gamma_{\Phi}$ and the canonical local trivialisation $\phi$ with respect to the local section $\sigma$ as
    \begin{equation}
        \Psi(p,v) = \varphi^{-1} \left( \Gamma_{\Phi} (1,\varphi(p,v)) \right)\,
    \end{equation}
    for any $p \in M$ and $v \in V$. Note that this indeed defines a diffeomorphism, since both $\varphi$ and $\Gamma_{\Phi}$ are diffeomorphisms. We also note that the left action of $G$ on $M \times V$ is given by~\eqref{eq:left_action_on_M_times_V} as,
    \begin{equation}
        L_g(p,v) = \varphi^{-1} \left( L_g ( \varphi(p,v)) \right)\,,
    \end{equation}
    for any $p \in M$ and $v \in V$. Consequently,~\Cref{lem:equiv_lift_ass} implies
    \begin{eqnarray*}
        L_g \circ \Psi &=& \left(\varphi^{-1} \circ L_g \circ \varphi \right) \circ \left( \varphi^{-1} \circ \Gamma^1_{\Phi} \circ \varphi \right) = \varphi^{-1} \circ L_g \circ \Gamma^1_{\Phi} \circ \varphi \\
        &=& \varphi^{-1} \circ \Gamma^1_{\Phi} \circ L_g \circ \varphi = \left(\varphi^{-1} \circ \Gamma^1_{\Phi} \circ \varphi \right) \circ \left( \varphi^{-1} \circ L_g \circ \varphi \right) \\
        &=& \Psi \circ L_g\,,
    \end{eqnarray*}
    where we use the notation $\Gamma^1_{\Phi}([p,v]) = \Gamma_{\Phi}(1,[p,v])$. This completes the proof.
\end{proof}

\begin{remark}
    Note that, just as in the definition of steerable NODEs in~\Cref{dfn:steerable_NODE}, the formulation in~\Cref{thm:equivariance_steerable_NODE_flow} is local in the sense that $G \times_{\rho} V$ is locally, but not in general globally, diffeomorphic to $M \times V$. However, the construction is independent of the choice of local section $\sigma$, which is implicit in~\Cref{thm:equivariance_steerable_NODE_flow}, and can be extended globally using the transition functions defining $G$ as a principal $H$-bundle.  
\end{remark}

We note that the converse of~\Cref{thm:equivariance_steerable_NODE_flow} is not true. Equivariance of the steerable NODE flow implies invariance of the vector field $\phi$ generating the corresponding flow on $M=G/H$. However, it does not guarantee invariance of $\omega$, only that a one-dimensional subspace of the horizontal component $H_gG \subset T_gG$ is preserved by the left action $L_g$. This is illustrated by the following example.
\begin{example}
    We return to the setting in~\Cref{ex:steerabl_NODE}, where $G = \R^2 \times U(1)$ is the trivial $H=U(1)$ bundle over $M=G/H=\R^2$ with the global $G$-action given by translations. Let $\omega = d\theta + f(y)dy$, for some function $f:\R \to \R$, which is clearly a 1-form on $G$ taking values in $\mathfrak{h} = \R$. This is a principal connection on $G$, because:
    \begin{enumerate}
        \item[(i)] It acts trivially on the vertical subspace, since
        \begin{equation}
            \omega(\partial_{\theta}) = d\theta(\partial_{\theta}) + f(y)dy(\partial_{\theta}) = 1 + f(y)\cdot0 = 1 \,.
        \end{equation}
        \item[(ii)] It is $H$-equivariant, since with $h=(0,0,\varphi) \in H$ we have $R_h(x,y,\theta) = (x,y,\theta+\varphi \, (\mathrm{mod}\, 2\pi))$ which implies
        \begin{equation}
            (R_h)^*d\theta = d\theta \,, \quad (R_h)^*(f(y)dy) = f(y)dy \,. 
        \end{equation}
        Consequently, $(R_h)^*\omega = \omega = \mathrm{Ad}_{h^{-1}}\omega$, where we have used that $\mathrm{Ad}_{h}$ is trivial for $H$ abelian.
    \end{enumerate}
    However, the connection $\omega$ is not $G$-invariant in general, since with $g = (0,\beta,0) \in G$ we have $L_g(x,y,\theta) = (x,y+\beta,\theta)$ and
    \begin{equation}
        (L_g)^*\omega = d\theta + f(y+\beta)dy \,.
    \end{equation}
    Thus, $\omega$ is invariant only if $f$ is constant. Taking $f$ non-constant gives a principal connection which is not $G$-invariant.

    We now construct a horizontal, $G$-equivariant flow on $G$ that projects to a $G$-equivariant flow on $M=G/H$. Let $\tilde{\Phi}$ be the flow generated by $\tilde{\phi} = \partial_x$. From~\eqref{eqn:counterex_left_action_G}, we have $(L_g)_*\tilde{\phi} = \tilde{\phi}$ so $\tilde{\phi}$ is $G$-invariant and $\tilde{\Phi}$ is $G$-equivariant. Furthermore, we have
    \begin{equation}
        \omega(\tilde{\phi}) = \omega(\partial_x) = d\theta(\partial_x) + f(y)dy(\partial_x) = 0 \,,
    \end{equation}
    which means that $\tilde{\phi}$ is horizontal with respect to $\omega$. Finally, $\tilde{\phi}$ projects to $\phi=\pi_*\tilde{\phi} = \pi_*(\partial_x) = \partial_x$. From~\eqref{eqn:counterex_left_action_M}, we have $(L_g)_*\phi = \phi$, so $\phi$ is a $G$-invariant vector field on $M$ generating a $G$-equivariant flow $\Phi$. Consequently, $\tilde{\Phi}$ has the desired properties.

    Extending to the parallel transport flow $\Gamma_{\Phi}$ on the associated bundle $G \times_{\rho} V$, with $\rho:U(1) \to GL(V)$, we have
    \begin{equation}
        \Gamma_{\Phi}(t,[g,v]) = [\tilde{\Phi}(t,g),v]\,.
    \end{equation}
    Equivariance of $\tilde{\Phi}$ implies equivariance of $\Gamma_{\Phi}$; see the proof of~\Cref{lem:equiv_lift_ass}. Moreover, $\Gamma_{\Phi}$ is horizontal by construction. Consequently, the steerable NODE $\Psi$ obtained from $\Gamma_{\Phi}$ is $G$-equivariant.
    
    This illustrates why the converse of~\Cref{thm:equivariance_steerable_NODE_flow} is not true. Equivariance of the flow $\tilde{\Phi}$ in the $x$-direction forces $\omega$ to be invariant under the corresponding translations, but it can still fail to be invariant under $y$-translations and therefore under the full group $G$.
\end{example}

\subsection{Equivariance in the local formulation}\label{sec:local_form_of_equiv}
We now consider equivariance of steerable NODEs in their local description, i.e., maps $\Psi:M\times V\to M\times V$ defined by
\begin{equation}
    \Psi(p,v)=(\Phi_p(1),\rho(h_p(1))v)\,,
\end{equation}
for all $(p,v)\in M\times V$, where $\Phi_p$ and $h_p$ are given by~\cref{dfn:steerable_NODE} for some vector field $\phi$ and connection $\omega$. Since $\Psi$ acts on the space $M\times V$, we must consider the left action on $M\times V$. Recall that according to~\eqref{eq:left_action_on_M_times_V} this action is given by
\begin{equation}
    L_g(p,v)=(gp,\rho(c(g,p))v),
\end{equation}
where $c:G\times M\to H$ is the map defined in \eqref{eq:cocycle}. In the following \namecref{lem:local_equiv}, we show that the equivariance of $\Psi$ under this action can be characterised by the equivariance of the vector field $\phi$ together with a condition involving the map $c$.

\begin{lemma}\label{lem:local_equiv}
    Let $\Psi:M\times V\to M\times V$ be a steerable NODE defined by the vector field $\phi: M\to TM$ and the principal $H$-connection $\omega$ on $G$. Then $\Psi$ is $G$-equivariant if and only if the vector field $\phi$ is $G$-invariant and
    \begin{equation}\label{eq:equivariance_vector_comp}
        c(g,\Phi_p(1))\, h_p(1)
        = h_{g p}(1)\, c(g,p)\, k(g,p)
    \end{equation}
    holds for all $g\in G$, $p\in M$ and some $k(g,p) \in \ker \rho$.
\end{lemma}
\begin{proof}
    Using the definition of the left action of $G$ on $M\times V$ in~\eqref{eq:left_action_on_M_times_V}, we compute
    \begin{equation}
        L_g \Psi(p, v) = (L_g \Phi_p(1), \rho(c(g, \Phi_p(1)) h_p(1)) v)
    \end{equation}
    and
    \begin{equation}
        \Psi(L_g(p, v)) = \big(\Phi_{gp}(1), \rho(h_{gp}(1)c(g,p)) v \big)
    \end{equation}
    for $g \in G$, $p \in M$, and $v \in V$. These expressions agree for all $g \in G$, $p \in M$, and $v \in V$ if and only if $\phi$ is $G$-invariant and $\rho(c(g, \Phi_p(1)) h_p(1)) v=\rho(h_{gp}(1)c(g,p)) v$ holds for all $g\in G$, $p\in M$ and $v\in V$. The second condition is equivalent to
    \begin{equation}
        c(g, \Phi_p(1)) h_p(1) \left(h_{gp}(1)c(g,p)\right)^{-1}\in\ker{\rho},
    \end{equation}
    yielding~\eqref{eq:equivariance_vector_comp}.
\end{proof}

The local description of equivariance provided in~\cref{lem:local_equiv} can be visualised by the commutative diagram in~\cref{fig:commutative_diagram}.

\begin{figure}[!t]
\centering
\includegraphics[scale=0.6]{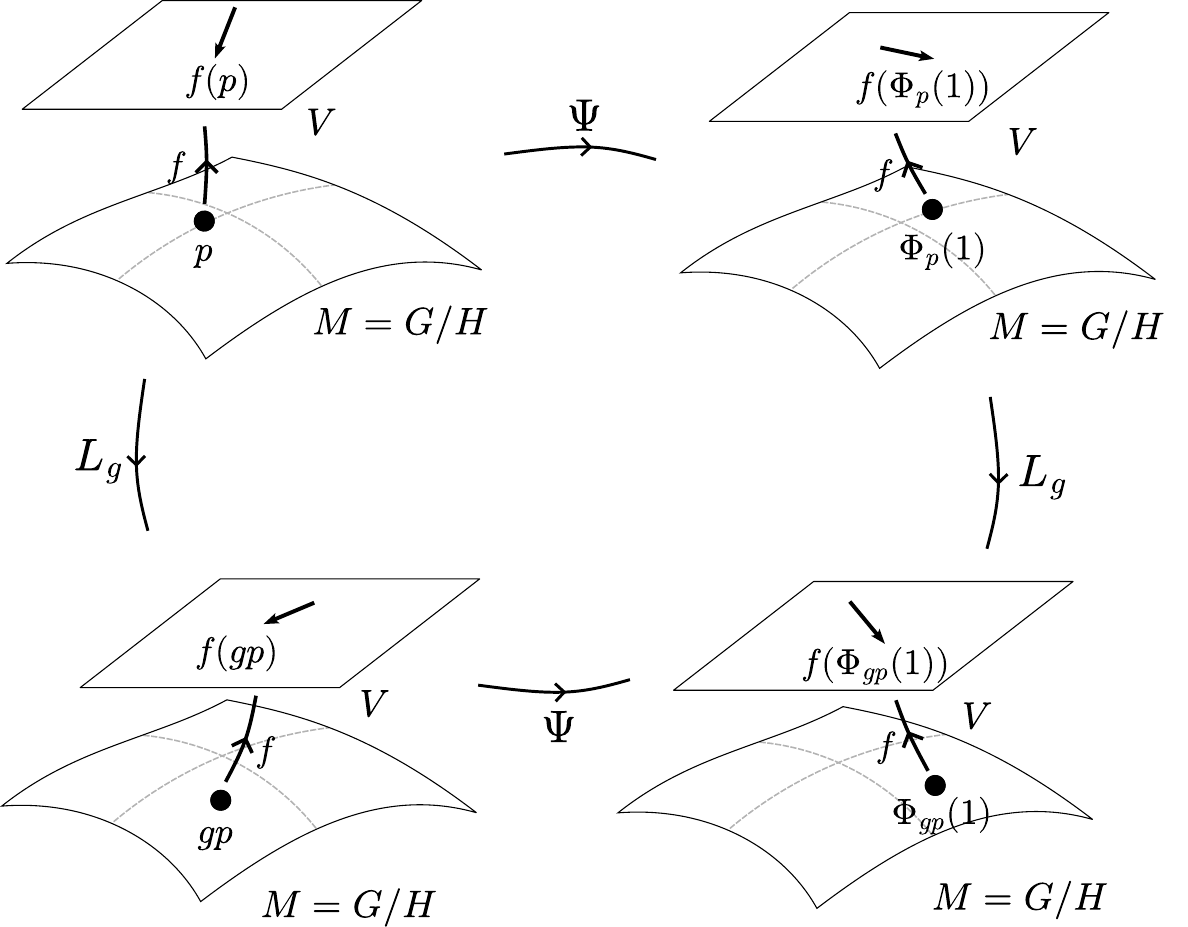}
\caption{An equivariant steerable NODE has the following commutative diagram. It is clear from the diagram that $L_g\Phi_p(1)=\Phi_{gp}(1)$ must hold. Moreover, it follows from how features $f(p)$ are transformed under $\Psi$ (\cref{dfn:steerable_NODE}) and how $L_g$ acts on local feature fields (\Cref{tab:feature_field_transformation}) that $f(\Phi_p(1))=\rho(h_p(1))f(p)$, $f(gp)=\rho(c(g,p))f(p)$, and $f(\Phi_{gp}(1))=\rho(c(g,\Phi_p(1))h_p(1))f(p)=\rho(h_{gp}(1)c(g,p))f(p)$, where the latter condition is \eqref{eq:equivariance_vector_comp}. Compare with~\cite[Figure 8]{Andersdotter2025}.}
\label{fig:commutative_diagram}
\end{figure}

The local equivariance condition in~\Cref{lem:local_equiv} is automatically satisfied under the geometric assumptions of~\cref{thm:equivariance_steerable_NODE_flow}. Suppose $\phi$ and $\omega$ are $G$-invariant. From the construction of $\Psi$, the curves $h_p$ and $h_{gp}$ satisfy $h_p(0) = h_{gp}(0) = e$ and define two unique horizontal lifts $\tilde{\Phi}_p(t) = \sigma(\Phi_p(t)) h_p(t)$ and $\tilde{\Phi}_{gp}(t) = \sigma(\Phi_{gp}(t)) h_{gp}(t)$ over $\Phi_p$ and $\Phi_{gp}$, respectively. 

Since $\omega$ is $G$-invariant, the curve $L_g\tilde{\Phi}_p$ must also be horizontal for all $g\in G$ and $p\in M$. Thus, for each $p$ and $g$, there exists a unique element $h(g,p)\in H$ such that $L_g \tilde{\Phi}_p(t) = \tilde{\Phi}_{gp}(t)h(g,p)$ for all $t$.

Left translation gives
\begin{equation}
    L_g \tilde{\Phi}_p(t) 
= \sigma(L_g \Phi_p(t)) c(g, \Phi_p(t)) h_p(t)\,.
\end{equation}
In particular, we have
\begin{equation}\label{eq:eq_at_0}
    L_g\tilde{\Phi}_p(0)=\tilde{\Phi}_{gp}(0)c(g,p)\,,
\end{equation}
which implies $h(g,p)=c(g,p)$. Substituting $\tilde{\Phi}_p(t) = \sigma(\Phi_p(t)) h_p(t)$ and $\tilde{\Phi}_{gp}(t) = \sigma(u_{gp}(t)) h_{gp}(t)$ yields
\begin{equation}
    c(g, \Phi_p(t))\, h_p(t) = h_{gp}(t)\, c(g,p).
\end{equation}
For $t = 1$, this corresponds to~\eqref{eq:equivariance_vector_comp} with $k=e$.

We conclude this section with an example that illustrates the local description of equivariance.

\begin{example}
    We return again to~\Cref{ex:steerabl_NODE}, and the steerable NODE defined by $\phi = \partial_x$ and $\omega = \pi dx + dy + d\theta$. Invariance of $\phi$ and $\omega$ is immediate since $G$ acts by translations. To determine whether $\Psi(x,y,v) = (x+1,y,R(\pi)v)$ is equivariant, in the sense of~\cref{dfn:equivariant_steerable_NODE}, we need to determine the function $c(g,p)$ for the section $\sigma(x,y) = (x,y,\chi)$. With $g = (\alpha,\beta,\varphi) \in G$, we find $c(g,p) = \varphi$, independently of $p \in M$.

    For $p = (x,y) \in \R^2$ and $v \in \R^2$, we then have
    \begin{equation}
        L_g (p,v) = (g p,\rho(c(g,p))v) = (x+\alpha,y+\beta,R(\varphi)v) \,,
    \end{equation}
    and consequently
    \begin{equation}
        \Psi(L_g(p,v)) = (x+\alpha+1,y+\beta,R(\pi)R(\varphi)v)\,.
    \end{equation}
    Interchanging the order of $\Psi$ and $L_g$, we obtain
    \begin{equation}
        L_g(\Psi(p,v)) = L_g(x+1,y,R(\pi)) = (x+1+\alpha,y+\beta,R(\varphi)R(\pi)v) \,.
    \end{equation}
    Clearly, since $H$ is abelian, $R(\varphi)R(\pi) = R(\pi)R(\varphi)$ and we have $L_g \circ \Psi = \Psi \circ L_g$ as expected from~\cref{thm:equivariance_steerable_NODE_flow}.

    We can also verify equivariance of $\Psi$ by checking the conditions in~\Cref{lem:local_equiv}. The vector field $\phi$ is invariant by construction and $c(g,p)=\varphi$. The map $h_p$ was computed in~\cref{ex:steerabl_NODE} to be $h_p(t)=-\pi t$, independent of the point $p$. Thus
    \begin{equation}
        c(g,\Phi_p(t))h_p(t)=h_{gp}(t)c(g,p)
    \end{equation}
    holds, which is the condition in~\eqref{eq:equivariance_vector_comp} with $k=e$.
\end{example}

\subsection{Invariant connections and Wang's theorem}
\label{sec:Wangs_theorem}
We have seen that steerable NODEs are equivariant if both the vector field $\phi$ and the connection $\omega$ are invariant. When we treat these as learnable quantities, we therefore need to understand how to parametrise the spaces of invariant vector fields and principal connections. In~\cite{Andersdotter2025}, it was shown that invariant vector fields on $M$ can be parametrised using the differential invariants of the action of $G$ on $M$. For invariant principal connections on $G$, a corresponding parametrisation can be accomplished using a theorem by Wang. Proposition A of~\cite{Wang1958} provides, as a special case, the following classification of $G$-invariant connections on $G$ as a principal $H$-bundle over $M=G/H$.

\begin{restatable}[Wang's theorem~\cite{Wang1958}]{theorem}{WangTheorem}\label{thm:wang}
    The $G$-invariant connections $\omega$ on $G$ are in one-to-one correspondence with linear maps $\Lambda:\mathfrak{g} \to \mathfrak{h}$ that satisfy
    \begin{itemize}
        \item[(i)] $\Lambda \circ \mathrm{Ad}_h = \mathrm{Ad}_h \circ \Lambda, \quad \forall h \in H$,
        \item[(ii)] $\left.\Lambda\right|_{\mathfrak{h}} = \left.\mathrm{id}\right|_\mathfrak{h}$,
    \end{itemize}
    where $\mathrm{Ad}$ is the adjoint action of $H$ on its Lie algebra $\mathfrak{h}$.
\end{restatable}

More concretely, the theorem provides explicit constructions of the linear map $\Lambda$ and invariant connection $\omega$. Given a $G$-invariant connection $\omega$, the associated linear map is defined as
\begin{equation}
    \Lambda := \omega_{e}\,.
\end{equation}
Conversely, given the linear map $\Lambda$, the corresponding invariant connection is defined by
\begin{equation}
    \omega_g:=(L_{g^{-1}})^*\Lambda \,,
\end{equation}
for all $g\in G$. In~\cref{app:proof_of_wang}, we provide a self-contained proof of~\cref{thm:wang}.

We will now illustrate how this construction allows us to parametrise the most general invariant connections in two familiar examples.

\begin{example}\label{ex:wang_R2}
We first consider the simplest possible example $G=\mathbb{R}^2$, $H=\mathbb{R}$, $M=G/H=\mathbb{R}$, where both left and right actions are translations. The Lie algebras are $\mathfrak{g}=\mathbb{R}^2$, $\mathfrak{h} = \mathbb{R}$, and the adjoint action of $H$ is trivial since it is abelian, $\mathrm{Ad}_h=\id_{\mathfrak{h}}$. We write an element of $\mathfrak{g}$ as $X=[\alpha\,\,\beta]^T$ and an element of $\mathfrak{h}$ as $Y=\beta$, and their images under the exponential map as
\begin{equation}
    g = \exp(X) = (\alpha,\beta) \,,\quad h = \exp(Y) = \beta\,.
\end{equation}

We now construct all $G$-invariant principal connections on $G$ using Wang's theorem. A linear map $\Lambda:\mathfrak{g} \to \mathfrak{h}$ is parametrised as matrix $\Lambda = [a\,\,b]$ and acts on $X$ as
\begin{equation}
    \Lambda(X) = [a \,\, b][\alpha \,\, \beta]^T = a\alpha+b\beta \,.
\end{equation}
We want to describe the most general map $\Lambda$ satisfying the conditions on $\Lambda$ in~\cref{thm:wang}. Condition (i) is trivial since $\mathrm{Ad} = \id_{\mathfrak{h}}$. Acting on $X = [0\,\,\beta]^T \in \mathfrak{h} \subset \mathfrak{g}$ yields
\begin{equation}
    \Lambda(X) = [a\,\,b][0\,\,\beta]^T = b\beta\,.
\end{equation}
Condition (ii) amounts to $\Lambda(X) = \beta$, which is equivalent to $b=1$. Consequently, the linear maps satisfying both (i) and (ii) are
\begin{equation}
\label{eqn:wang_maps_R2_example}
    \Lambda = [a\,\,1] \,,\quad a\in\mathbb{R}\,,
\end{equation}
which is a 1-parameter family of maps. The corresponding 1-parameter family of $G$-invariant connections is obtained by first defining
\begin{equation}
    \omega_e := \Lambda.
\end{equation}
With $X = [\alpha\,\,\beta]^T$, we then have
\begin{equation}
    \omega_e(X) = \Lambda(X) = [a\,\,1][\alpha\,\,\beta]^T = a\alpha+\beta \in \mathfrak{h}\,.
\end{equation}
Subsequently, we define $\omega \in \Omega^1(G,\mathfrak{h})$ as
\begin{equation}
\label{eqn:inv_connection_R2_example}
    \omega_g := (L_g^{-1})^* \omega_e, \ \ \ \forall g\in G\,.
\end{equation}
Because $G$ acts as translations we have $(L_g)^* = \id$, implying that $\omega$ is a constant form on $G$.

We conclude this example by providing some intuition for the 1-parameter family of invariant connections defined by~\eqref{eqn:inv_connection_R2_example} and~\eqref{eqn:wang_maps_R2_example}. To this end, we first consider the horizontal subspace defined by the connection $\omega$ corresponding to the linear map $\Lambda = [a\,\,1]$. At the identity, this subspace is defined as the kernel $H_eG =\ker \omega_e$, which for $v=[\alpha\,\,\beta]$ amounts to
\begin{equation}
    \omega_e(v) = a\alpha + \beta = 0\,.
\end{equation}
Consequently, $H_eG \subset T_eG \cong \mathbb{R}^2$ is the straight line $\beta = -a \alpha$, which corresponds to the expected 1-parameter family of possible decompositions $T_eG = H_eG \oplus V_eG$ obtained by specifying a slope $a$ of $H_eG$. Since the connection $\omega$ is constant, the decomposition $T_gG = H_gG \oplus V_gG$ is identical for all $g \in G$, corresponding to a constant slope throughout $G$. As expected, this defines an invariant decomposition of the tangent bundle $TG$.

Finally, we consider the vector field $\phi = \partial_x$ on $M$, which is $G$-invariant by virtue of $(L_g)_* = \id$. The integral curve through the origin in $M$ is $\Phi_0(t) = t$ and its horizontal lift through $(0,0) \in G$ is $\tilde{\Phi}_0(t) = (t,-at)$, since
\begin{equation}
    \dot{\tilde{\Phi}}_0(t) = \frac{d}{dt}\tilde{\Phi}_0(t) = [1 \,\, -a]
\end{equation}
which satisfies $\Lambda(\dot{\tilde{\Phi}}_0(t)) = [a\,\,1][1\,\,-a]^T = 0$. The parallel transport of $(0,0)$ along $u(t)$ is consequently $(1,-a)$. The generalisation to an arbitrary initial point $(x,y) \in G$ is straightforward.
\end{example}

\begin{example}\label{ex:wang_SO3}
We consider the case $G=\SO(3)$, $H=\SO(2)$ and $M=G/H=\SO(3)/\SO(2)\cong S^2$. The canonical projection $\pi:G\to M$ sends a matrix $R\in \SO(3)$ to its first column. To make $\pi$ invariant under the right action of $H$, we embed $H$ in $G$ as the subgroup of rotations about the first coordinate axis, i.e.,
\begin{equation}
    H \hookrightarrow G, \qquad
    \begin{pmatrix}
    a_{11} & a_{12} \\
    a_{21} & a_{22}
    \end{pmatrix} \mapsto
    \begin{pmatrix}
    1 & 0 & 0 \\
    0 & a_{11} & a_{12} \\
    0 & a_{21} & a_{22}
    \end{pmatrix}\,.
\end{equation}

The Lie algebra $\mathfrak{g}=\mathfrak{so}(3)$ is spanned by
\begin{equation}
    X_1 = \begin{pmatrix}
    0 & 0 & 0 \\
    0 & 0 & -1 \\
    0 & 1 & 0
    \end{pmatrix}, \ \
    X_2 = \begin{pmatrix}
    0 & 0 & 1 \\
    0 & 0 & 0 \\
    -1 & 0 & 0
    \end{pmatrix}, \ \
    X_3 = \begin{pmatrix}
    0 & -1 & 0 \\
    1 & 0 & 0 \\
    0 & 0 & 0
    \end{pmatrix},
\end{equation}
while the Lie algebra $\mathfrak{h}$ is spanned by $X_1$. Thus, a general element of $\mathfrak{g}$ can be written as
\begin{equation}
    X=a_1X_1+a_2X_2+a_3X_3, \qquad a_1,a_2,a_3\in\R,
\end{equation}
and an element in $\mathfrak{h}$ has the form $Y = bX_1$, where $b\in\R$. The exponentiation of $Y$ gives a general element in $H$ of the form
\begin{equation}
    h = \exp(Y) = \exp(bX_1)=
    \begin{pmatrix}
        1 & 0 & 0 \\
        0 & \cos(b) & -\sin(b) \\
        0 & \sin(b) & \cos(b)
    \end{pmatrix}.
\end{equation}

A general linear mapping $\Lambda:\mathfrak{g}\to\mathfrak{h}$ is given by
\begin{equation}
    \Lambda(X) = \Lambda(a_1X_1+a_2X_2+a_3X_3)=(a_1c_1+a_2c_2+a_3c_3)X_1,
\end{equation}
where $c_1,c_2,c_3\in\R$ are constants defining $\Lambda$. To impose the conditions in~\cref{thm:wang}, we first compute the adjoint action $\mathrm{Ad}_h:\mathfrak{g}\to\mathfrak{g}$ defined by
\begin{equation}
    \mathrm{Ad}_h(X)=hXh^{-1}\,,
\end{equation}
where we use matrix notation for the induced action of $H$ on $\mathfrak{g}$.

Using orthogonality, we immediately obtain $\mathrm{Ad}_h(X_1)=X_1$, $\mathrm{Ad}_h(X_2)=\cos(b)X_2+\sin(b)X_3$, and $\mathrm{Ad}_h(X_3)=-\sin(b)X_2+\cos(b)X_3$. Thus,
\begin{equation}
    \mathrm{Ad}_h(X) = a_1X_1 + (a_2\cos(b)-a_3\sin(b))X_2 + (a_2\sin(b)+a_3\cos(b))X_3.
\end{equation}
Applying $\Lambda$ gives
\begin{equation}\label{eq:Lambda_Ad}
    \Lambda(\mathrm{Ad}_h(X))=\left(a_1c_1 + (a_2\cos(b)-a_3\sin(b))c_2 + (a_2\sin(b)+a_3\cos(b))\,c_3\right)X_1.
\end{equation}
Moreover, since $\mathrm{Ad}_h$ acts trivially on $\mathfrak{h}$,
\begin{equation}\label{eq:Ad_Lambda}
    \mathrm{Ad}_h(\Lambda(X))=\Lambda(X)=(a_1c_1+a_2c_2+a_3c_3)X_1 \,.
\end{equation}
For condition (i) in~\cref{thm:wang} to hold, \eqref{eq:Lambda_Ad} and \eqref{eq:Ad_Lambda} must be equal for all $h$, forcing $c_2=c_3=0$. Condition (ii) of~\cref{thm:wang} requires that $\Lambda$ restricted to $\mathfrak{h}$ is the identity mapping. Since $\Lambda(X_1)=c_1X_1$, this implies $c_1=1$.

We conclude that there is a unique linear map $\Lambda:\mathfrak{g}\to\mathfrak{h}$ satisfying both conditions in~\cref{thm:wang}:
\begin{equation}
    \Lambda(a_1X_1+a_2X_2+a_3X_3)=a_1X_1\,.
\end{equation}
By~\cref{thm:wang}, this determines a unique $G$-invariant connection $\omega$ given by
\begin{equation}
    \omega_g=(L_{g^{-1}})^*\Lambda\,,
\end{equation}
for each $g\in G$.
\end{example}

\section{Relation to existing NODE models}\label{sec:existingNODEs}
Having defined steerable NODEs on homogeneous spaces $M=G/H$ and derived conditions for equivariance under the global symmetry group $G$, we now turn to the relation of these models to the existing literature on NODEs and their applications. First, we consider the relation to previous constructions where an action on tangent vectors on $M$ is induced by the flow. Second, we show how our steerable NODEs extend and generalise existing work on continuous normalizing flows on Lie groups. Finally, we provide a detailed discussion of steerable NODEs in the specific case $M = S^2$ to illustrate our framework in a familiar geometric setting common in applications.

\subsection{Parallel transport and the push-forward}\label{sec:parallel_transport_push_forward}
In our previous work~\cite{Andersdotter2025}, we showed that a manifold NODE $\psi:M\to M$ induces an action on tangent vector fields $X:M\to TM$ through the push-forward. This is analogous to the construction of CNFs, which instead induce an action through the push-forward of densities. To demonstrate how the induced action on vector fields is related to steerable NODEs, in the setting of homogeneous spaces, we now show an example of how the push-forward can be expressed as parallel transport along a flow for a specific connection $\omega$. 

\begin{example}\label{ex:push_forward_R2}
We return to the setting in~\cref{ex:wang_R2}, where $G=\R^2$ acts by translations on $M=G/H = \R$ and let $\psi:M\to M$ be the manifold NODE defined by the vector field $\phi = \partial_x$. Since $(L_g)_* = \id$, the vector field $\phi$ is invariant, and the manifold NODE $\psi$ is equivariant. The integral curves of $\phi$ are $\Phi_x(t) = x+t$, for every $x\in M$, and, consequently, we have $\Psi(x) = x+1$. Applying the action induced by the push-forward of $\psi$ to a vector $w \in T_xM \cong \R$ in the tangent space to $M$ yields
\begin{equation}
\label{eqn:push_forward_ex_push_forward}
    \psi_*(w) = d\psi(w) = w \,.
\end{equation}

To express the induced action using a steerable NODE $\Psi$, we first take $V \cong \R$ so that the feature vectors have the correct dimension to be able to represent tangent vectors. In the steerable NODE context, the equivalent condition to~\eqref{eqn:push_forward_ex_push_forward}, that the transformation of the feature vector should be trivial, must be accompanied by a choice of section $\sigma:M\to G$ of the principal bundle relative to which the vector component should be constant. This is due to the fact that, unlike the situation for the push-forward acting on the tangent space, there is no canonical reference in the vector space $V$ in the steerable NODE setting\footnote{This can be understood as the vector space $V \cong \R$ being disconnected from its geometrical origin representing tangent vectors. Establishing the connection between $V$ and $T_xM$ requires a choice of section relating the two.}.

We consider the section
\begin{equation}
\label{eqn:push_forward_ex_section}
    \sigma(x) = (x,f(x)) \,, \quad x \in M \,,
\end{equation}
where $f:\R \to \R$ is a smooth function. With $\phi = \partial_x$, we then have
\begin{equation}
    \sigma_*\phi_{\Phi_x(t)} = (1,f'(x+t)) \,,
\end{equation}
and, using $(R_h)_* = \id$ for translations $h \in H$, we obtain the parallel transport equation
\begin{equation}
\label{eqn:push_forward_ex_parallel_transport}
    \frac{d h_x(t)}{dt} = -\omega(\sigma_* \phi_{\Phi_x(t)}) = -(a(x+t) + f'(x+t)) \,, \quad h_x(0)=0\,.
\end{equation}
Here, we have used the fact that any connection $\omega$, corresponding to a decomposition of the tangent space of the total space $G$, can be expressed as $\omega_g = [a(\pi(g))\,\,1]$, $g \in G$, for some smooth function $a : M \to \R$. This decomposition is equivariant under right translations by $h \in H$, as required, since $a$ depends only on the point $\pi(g)$ in $M$.

The condition that the steerable NODE preserves the feature vector (relative to the section $\sigma$) then amounts to $h_x(1) = 0$. From~\eqref{eqn:push_forward_ex_parallel_transport}, we see that this condition is in particular satisfied if
\begin{equation}
    a(x) = -f'(x) \,, \quad \forall x\in M\,.
\end{equation}
Furthermore, if the section is $G$-invariant, $\sigma(gx) = \sigma(x)$ for any $g \in G$, we have $a(x)=-f'(x)=0$ for any $x \in M$, which corresponds to an invariant connection, as expected.
\end{example}

\subsection{Continuous normalizing flows on $G/H$}
In this section, we demonstrate that our framework captures continuous normalizing flows on Lie groups. In particular, we connect to the normalizing flows for lattice gauge theories studied in \cite{Gerdes2024}. We demonstrate that our framework provides a generalisation of this to non-trivial stabiliser subgroups $H\subset G$.

\subsubsection{Continuous normalizing flows on a Lie group $G$}
Continuous normalizing flows (CNFs) provide a framework for learning probability densities via diffeomorphic transformations generated by NODEs. Recently, CNFs have been applied to lattice gauge theories, where the data take values in Lie groups and are subject to gauge symmetries~\cite{Gerdes2024}. In this section, we show that these constructions can be interpreted as a special case of steerable neural ODEs on homogeneous spaces, corresponding to the situation where the stabiliser subgroup $H$ is trivial. 

Let $G$ be a compact Lie group equipped with Haar measure $\mu_G$, and let $p_0 : G \to \mathbb{R}_{\ge 0}$ be an initial probability density. In lattice gauge theories, one usually takes $G=SU(N)$, but there is no need to make that restriction here. A CNF on $G$ is defined by a vector field $\phi : G \to TG$ and the associated Cauchy problem
\begin{equation}\label{eqn:CNF-ODE}
    \dot g(t) = \phi(g(t))\,, \qquad g(0) = g_0 \in G\,, 
\end{equation}
where $g(t)\in G$ is the solution curve of the ODE on the group $G$, generated by $\phi$. In the notation of section \ref{sec:steerable_NODEs} we have $g(t)=\Phi_{g_0}(t)$. 

The solution defines a flow $\Phi : \R \times G \to G$, and the time-evolved density $p_t$ satisfies the continuity equation
\begin{equation}\label{eqn:CNF-continuity}
    \frac{d}{dt} \log p_t(g(t)) = - \operatorname{div}_G \phi(g(t)),
\end{equation}
where $\operatorname{div}_G$ denotes the divergence with respect to the Haar measure $\mu_G$, defined by the equation $\mathcal{L}_{\phi}\mu_G = (\operatorname{div}_G \phi)\mu_G$.

In \cite{Gerdes2024}, the vector field $\phi$ is constructed to be equivariant under gauge transformations, ensuring that the induced flow preserves gauge symmetry of the target distribution. Importantly, the flow acts only on the group-valued variables themselves, without transporting additional feature fields. Moreover, to see the precise connection to \cite{Gerdes2024}, one should choose 
\begin{equation}  
    \phi(g(t))=Z_\theta(g)g(t)\,,
\end{equation}
where $Z_\theta : G \to \mathfrak{g}$ is a neural network with parameters $\theta$. If we set $G=SU(N)$ and fix some basis $\{T_a\}$ of the Lie algebra $\mathfrak{su}(N)$ we can write the divergence in \ref{eqn:CNF-continuity} explicitly as $\sum_a \partial_a Z^{a}_\theta(g)$. Moreover, in a lattice gauge theory setting, the group element $g$ would be replaced by a collection of elements $\{g_e\}_{e\in E}$, where $e \in E$ labels the edges of the lattice.

\subsubsection{Connection with NODEs on $G/H$}
This setting fits naturally into our framework by viewing the Lie group $G$ as a homogeneous space $M = G/H$ with $H = \{e\}$. In this case, the principal bundle $G \to G/H$ is trivial, and the base manifold is canonically identified with $G$ itself. A manifold neural ODE on $M$ in the sense of Definition~\ref{dfn:mfd_NODE} is therefore exactly an ODE of the form~\eqref{eqn:CNF-ODE}.

Since $H$ is trivial, there is no non-trivial local symmetry acting on the fibres, and the associated bundle with fibre $V$ reduces to the product bundle
\begin{equation}
G \times V \to G\,.
\end{equation}
Feature fields are simply maps $f : G \to V$, and there is no need to introduce a representation $\rho$ or a connection to ensure consistency under local gauge transformations.

From this perspective, the CNFs of \cite{Gerdes2024} correspond precisely to manifold NODEs on $G$, without steering of additional features. The equivariance constraints imposed on $\phi$ in that work ensure that the flow $\Phi$ respects the relevant global symmetries.

In the language of principal bundles, the vector fields used in~\cite{Gerdes2024} can be interpreted as left- or right-invariant vector fields on $G$, or more generally, as equivariant maps $\phi : G \to TG$ satisfying
\begin{equation}
    \phi(g g') = (L_g)_* \phi(g') \qquad \forall g,g' \in G \,.
\end{equation}
This is the condition required for the diffeomorphism induced by $\Phi$ to be $G$-equivariant in the sense of~\Cref{dfn:equivariant_mfd_NODE}. Since the bundle is trivial in this case, there is no distinction between vertical and horizontal directions, and no additional geometric structure is required to define the flow.

While the lattice gauge theory literature refers to this as ``gauge equivariance'' (stemming from the redundancy in the physical description), in our geometric framework on the homogeneous space $M=G$, this manifests as a global symmetry under the group action of $G$ on itself. This is precisely the condition required for the induced diffeomorphism $\Phi$ to be $G$-equivariant in the sense of~\Cref{dfn:equivariant_mfd_NODE}.

Although no explicit feature transport is introduced in standard CNFs, probability densities can be understood as geometric objects induced by the diffeomorphism $\Phi$. 
Equation~\eqref{eqn:CNF-continuity} arises from the Jacobian determinant of the flow and corresponds to the push-forward of the density under $\Phi$.
In this sense, CNFs already make implicit use of induced actions, but restricted to scalar-valued densities rather than general feature fields.

Our framework becomes richer when $H$ is non-trivial. In this case, the base manifold is a quotient space $M = G/H$, and features transform in a representation $\rho : H \to \mathrm{GL}(V)$. Feature fields are sections of the associated bundle $G \times_\rho V$, and their evolution along a NODE trajectory is governed by parallel transport with respect to a connection $\omega$. From a physics perspective, this corresponds to systems with internal degrees of freedom that transform under a local symmetry group. For instance, it could correspond to some gauge covariant matter fields that transform with respect to $H$. Another concrete interpretation might be a molecular dipole moment attached to a rotating rigid body, extending existing flow-based protein generation models~\cite{Yim2023,Bose2024,GarciaSatorras2021} to incorporate non-scalar features with consistent transformation properties.

In such settings, the base flow on $G/H$ describes the evolution of the physical configuration, while parallel transport in the associated bundle encodes how internal features are rotated or transformed consistently along the trajectory.

We conclude that CNFs on Lie groups correspond to the simplest case of steerable NODEs, where $H$ is trivial, and no internal feature transport is required. Our construction therefore provides a geometric generalisation of CNFs to homogeneous spaces with non-trivial stabilisers.

\subsection{Example: Steerable NODEs on $S^2$}
We now illustrate our framework in the special case of the homogeneous space
\begin{equation}
    M = \mathrm{SO}(3)/\mathrm{SO}(2) \simeq S^2 \,.
\end{equation}
This example highlights the role of invariant connections and clarifies the interpretation of steerable NODEs with non-trivial stabiliser subgroup. See also Example \ref{ex:wang_SO3} for more details on this setting. 

The group $G = \mathrm{SO}(3)$ acts transitively on the unit sphere $S^2$ by rotations. Fix a reference point $p_0 \in S^2$. The stabiliser subgroup of $p_0$ is $H = \mathrm{SO}(2)$. For example, $p_0$ can be chosen so that $SO(2)$ acts as rotations about the $x$-axis. This identifies $S^2$ with the homogeneous space $G/H$.

The principal bundle
\begin{equation}
    \pi : \mathrm{SO}(3) \to S^2
\end{equation}
maps a rotation matrix $R$ to its action on $p_0$, i.e., $\pi(R) = R p_0$. Geometrically, a point in the total space corresponds to a choice of local frame whose $x$-axis is aligned with the point on the sphere.

Let $V$ be a vector space carrying a representation
\begin{equation}
    \rho : \mathrm{SO}(2) \to \mathrm{GL}(V)\,.
\end{equation}
Feature fields on $S^2$ transforming under $\rho$ are sections of the associated bundle
\begin{equation}
    \mathrm{SO}(3) \times_\rho V \to S^2\,.
\end{equation}

To transport features along trajectories on $S^2$, we require a connection on the principal bundle $\mathrm{SO}(3) \to S^2$. Since $S^2$ is a homogeneous space, it is natural to restrict attention to $G$-invariant connections. Recall from~\Cref{sec:Wangs_theorem} that by Wang's theorem, $G$-invariant connections are in one-to-one correspondence with linear maps
\begin{equation}
    \Lambda : \mathfrak{so}(3) \to \mathfrak{so}(2)
\end{equation}
satisfying suitable equivariance and normalisation conditions. In this case, there exists a \emph{unique} $G$-invariant principal connection, corresponding to the reductive decomposition
\begin{equation}
    \mathfrak{so}(3) = \mathfrak{so}(2) \oplus \mathfrak{m},
\end{equation}
where $\mathfrak{m}$ is spanned by generators rotating $p_0$ into the tangent plane of $S^2$.

To proceed, let $\phi : S^2 \to TS^2$ be a (possibly equivariant) vector field defining a NODE on the sphere:
\begin{equation}
    \dot p(t) = \phi(p(t)), \qquad p(0) = p_0 \,.
\end{equation}
Given a horizontal lift $\tilde p(t) \in \mathrm{SO}(3)$ satisfying
\begin{equation}
    \pi(\tilde p(t)) = p(t)\,, \qquad \omega(\dot{\tilde p}(t)) = 0,
\end{equation}
the evolution of a feature vector $v \in V$ is given by parallel transport,
\begin{equation}
    \frac{d}{dt} v(t) = 0, \qquad f(p(t)) = [\tilde p(t), v(t)] \,.
\end{equation}
In local coordinates, this yields a transformation
\begin{equation}
    f(p(t)) = \rho(h(t)) f(p(0))\,,
\end{equation}
where $h(t) \in \mathrm{SO}(2)$ is determined by the connection and the base-space trajectory.

If $V = \mathbb{R}$ and $\rho$ is the trivial representation, the associated bundle reduces to a trivial line bundle and parallel transport is trivial. In this case, the steerable NODE reduces to a standard manifold NODE on $S^2$, recovering the setting of continuous normalizing flows on homogeneous spaces without internal structure. For non-trivial $\rho$, even though the connection is uniquely determined by invariance, the model captures genuinely new behaviour that cannot be expressed using base-space flows alone.

\section{Conclusions}
\label{sec:conclusions}

In this paper, we consider a geometric extension of neural ODEs on homogeneous spaces $M = G/H$, which incorporates features transforming in a representation $\rho$ of the stabiliser subgroup $H$. In this setting, a manifold NODE transports points along integral curves of a learnable vector field. Feature fields, on the other hand, are interpreted as sections of the associated bundle $G \times_\rho V$. We show that the natural mechanism for transporting such features along NODE trajectories is parallel transport with respect to a principal connection on the bundle $G \to G/H$. This leads to~\Cref{dfn:steerable_NODE} of a steerable NODE as a coupled geometric system consisting of a base flow on $G/H$ together with a parallel transport equation steering the feature field through a learnable principal connection.

The resulting construction provides a geometrically consistent notion of feature transport that generalises both the push-forward of tangent vectors and the density evolution appearing in continuous normalizing flows. A central structural result is that equivariance of the steerable NODE with respect to the global $G$-action is guaranteed if both the base vector field and the principal connection are $G$-invariant (\Cref{thm:equivariance_steerable_NODE_flow}). Using Wang’s theorem, we obtain an explicit classification of $G$-invariant principal connections on $G \to G/H$ which, together with our previous results for equivariant manifold NODEs, yields a concrete parametrisation of all admissible equivariant steerable NODE models.

Conceptually, our analysis shows that extending NODEs to general feature fields on homogeneous spaces necessarily introduces principal connections as learnable geometric structures. This reveals a fundamental distinction between scalar transport (e.g., densities) and general vector-valued feature transport: while scalar quantities may be transported canonically, non-scalar features require additional geometric data to define their evolution in a manner consistent with both local and global symmetries.

The framework developed here opens several directions for further research. In this work, we characterise invariant connections in order to ensure global equivariance of the steerable NODE. A natural next step is to investigate the learning dynamics of such connections in practice. In particular, it would be interesting to explore explicit parametrisations of invariant connections via Wang’s theorem for specific Lie groups and to understand the inductive bias induced by different classes of connections. Allowing for non-invariant connections may also be relevant in settings where symmetry is only approximate.

We emphasise that, although the global construction is well-defined, its local realisation depends on a choice of section of the principal bundle. In this sense, the formulation is gauge-dependent at the level of coordinates. A systematic investigation of gauge-equivariance, transition functions between local trivialisations, and numerically stable implementations across charts would further clarify the interplay between geometric structure and practical computation.

For equivariant manifold NODEs, universality results have been established under appropriate assumptions. An important open question is whether steerable NODEs with learnable connections retain comparable approximation properties for spaces of equivariant bundle maps or flows on associated bundles. While homogeneous spaces provide a clean and structured setting, many applications involve manifolds with local but not global homogeneous symmetry. Extending the present framework to more general principal bundles that do not arise from global quotients would significantly broaden its scope.

A particularly intriguing direction is to extend the present framework beyond smooth homogeneous spaces to discrete geometric settings such as combinatorial complexes. In recent years, graph and sheaf neural networks have demonstrated that feature fields on graphs and cell complexes can be modelled as sections of vector bundles \cite{Cassel2025} and cellular sheaves~\cite{Hansen2020,Barbero2022}, where restriction maps encode local consistency constraints analogous to parallel transport. From this perspective, cellular sheaves may be viewed as discrete counterparts of vector bundles, and their restriction maps as a combinatorial analogue of connection data. It would therefore be natural to investigate whether a theory of steerable NODEs can be formulated in this discrete setting, combining continuous-time dynamics on cochains with sheaf-consistent feature evolution. Such a development could provide a unified geometric framework bridging continuous gauge-theoretic models and discrete topological deep learning, and may open new avenues for modelling dynamical processes on graphs and higher-order networks.

In the more applied direction, we also showed that continuous normalizing flows on Lie groups arise as a special case of our construction when the stabiliser subgroup is trivial. This suggests the possibility of combining density evolution with feature transport, leading to gauge-equivariant generative models on homogeneous spaces. Such models may be of particular interest in physics-informed machine learning and geometric generative modelling.

Finally, it would be interesting to explore potential applications within geometric deep learning, such as data defined on spheres and rotation groups. This occurs for instance in molecular modelling (e.g.,~\cite{Yim2023,Bose2024}) or climate science (e.g.,~\cite{Linander2025}), as well as gauge-equivariant models (e.g., \cite{Favoni2022,Gerdes2024}) in lattice gauge theory. In these contexts, steerable NODEs offer a principled mechanism for coupling state evolution and feature transformation in continuous time while respecting symmetry constraints. 

We hope to report on progress on these interesting directions in the future. 

\section*{Acknowledgements}
The work of E.A., D.P., and F.O. was supported by the Wallenberg AI, Autonomous Systems and Software Program (WASP) funded by the Knut and Alice Wallenberg Foundation. The work of F.O. was partially funded by the Swedish Research Council under grant agreement no. 2025-05053. 

\bibliographystyle{unsrtnat}
\bibliography{cites}

\newpage
\appendix

\section{Proof of \Cref{lem:feature_fields}}\label{app:proof_of_equiv_feat_fields}
In~\cref{sec:feature_fields_and_induced_representations} we introduced two definitions of a feature field. Here, we prove the lemma stating that these two definitions are equivalent.
\MackeyLemma*
\begin{proof}
    We start by showing that a local feature field defines a section of the associated bundle. Let $\sigma:G/H\to G$ be a (local) section of the principal bundle $G\overset{\pi}{\to}G/H$, and let $f=k\circ\sigma$, where $k:G\to V$ is a Mackey function (\cref{dfn:mackey_fnc}). Define $s:M \to G\times_\rho V$ by
    \begin{equation}
        s(p) = \left[\sigma(p),f(p)\right] = \left[\sigma(p),k(\sigma(p))\right]\,,
    \end{equation}
    for $p\in M$. The Mackey condition ensures that $s$ is independent of the choice of principal section: if $\sigma'$ is another section with $\sigma'(p) = \sigma(p)h(p)$ for some $h(p)\in H$, then
    \begin{eqnarray*}
        \left[\sigma'(p),k(\sigma'(p))\right] &=& \left[\sigma(p)h(p),k(\sigma(p)h(p))\right] \\
        &=& \left[\sigma(p)h(p),\rho(h(p)^{-1})k(\sigma(p))\right] = \left[\sigma(p),k(\sigma(p))\right]\,.
    \end{eqnarray*}
    Thus $s$ is well-defined, smooth, and independent of the choice of local section $\sigma$.
    
    To show the converse, i.e., that every section $s$ of the associated bundle defines a Mackey function, let $\{U_i\}$ be an open cover of $M=G/H$ and denote by $\sigma_i:U_i\to G$ the sections defined locally on $U_i$, related by the bundle cocycles $c_{ij}:U_i\cap U_j\to H$ satisfying $\sigma_i(p)=\sigma_j(p)c_{ij}(p)$ for all $p\in U_i\cap U_j$. For all $g\in\pi^{-1}(U_i)$, there is a unique element $h_i(g)\in H$ such that $g=\sigma_i(\pi(g))h_i(g)$.

    Let $s:M\to G\times_\rho V$ be a smooth section. Then for each $p\in U_i$
    \begin{equation}\label{eq:smooth_section_local}
        s(p)=[\sigma_i(p),v_i(p)]\,,
    \end{equation}
    where $v_i(p)\in V$. Now, define $k:G\to V$ by
    \begin{equation}
        k(g):=\rho(h_i(g)^{-1})v_i(\pi(g))
    \end{equation}
    for each $g\in\pi^{-1}(U_i)$, where $h_i(g)$ is the unique element satisfying $g=\sigma_i(\pi(g))h_i(g)$. This is, indeed, a Mackey function, since the right action of an element $h\in H$ on $g$ corresponds to a new element in the same fibre:
    \begin{equation*}
        k(gh)=\rho(h^{-1})k(g)\,.
    \end{equation*}
    It is also well-defined on overlaps: If $p\in U_i\cap U_j$,
    \begin{equation}
        s(p)=[\sigma_i(p),v_i(p)]=[\sigma_j(p),v_j(p)]\,,
    \end{equation}
    from which a simple calculation shows that $v_j(p)=\rho(c_{ij}(p)^{-1})v_i(p)$. Furthermore, if $g=\sigma_i(p)h_i(g)$, there is also a unique element $h_j(g)\in H$ satisfying $g=\sigma_j(p)h_j(g)$. A similar calculation yields $h_i(g)=c_{ij}(p)h_j(g)$. It follows that
    \begin{equation}
        \rho(h_j(g)^{-1})v_j(p)=
        \rho(h_i(g)^{-1})\rho(c_{ij}(p))v_j(p)=\rho(h_i(g)^{-1})v_i(p)\,,
    \end{equation}
    showing that $k$ is well-defined on overlaps.

    With the definition of $k$, \eqref{eq:smooth_section_local} can be written as
    \begin{equation}
        s(p)=[\sigma_i(p),k(\sigma_i(p))]\,.
    \end{equation}
    If $p\in U_i\cap U_j$, this can be written as
    \begin{equation}
        s(p)=[\sigma_j(p)(c_{ij}(p))^{-1},k(\sigma_i(p))]=[\sigma_j(p),\rho((c_{ij}(p))^{-1})k(\sigma_i(p))]\,.
    \end{equation}
    Because $k$ is a Mackey function and $\sigma_i(p)c_{ij}(p)=\sigma_j(p)$, it follows that
    \begin{equation}
        s(p)=[\sigma_j(p),k(\sigma_j(p))]\,,
    \end{equation}
    showing that the definition of $k$ is consistent with the transformation behaviour of the section.

    The Mackey function $k$ is smooth on each $\pi^{-1}(U_i)$ because $h_i$, $\sigma_i$ and $v_i$ are smooth. Since they also agree on overlaps, $k$ is a globally smooth function.

    Lastly, note that if two Mackey functions $k$ and $k'$ are defined for the same section, then $k(\sigma_i(p))=k'(\sigma_i(p))$ for all $p\in U_i$. Since every element $g\in\pi^{-1}(U_i)$ satisfies $g=\sigma_i(\pi(g))h_i(g)$ for some unique $h_i(g)\in H$, the Mackey condition yields
    \begin{equation}
    k(g) = \rho(h_i(g)^{-1})k(\sigma_i(p)) = \rho(h_i(g)^{-1})k'(\sigma_i(p))=k'(g).
    \end{equation} 
    Thus, the Mackey function constructed from the section $s$ is unique.
\end{proof}

\section{Proof of Wang's theorem for $G \to G/H$}\label{app:proof_of_wang}
In this appendix, we provide a self-contained proof of~\cref{thm:wang}. The theorem is a special case of the more general result due to Wang~\cite{Wang1958}, in the situation when we consider $G$ as a principal $H$-bundle over the homogeneous space $G/H$. 

\WangTheorem*

\begin{proof}
    First, given a linear map $\Lambda:\mathfrak{g}\to\mathfrak{h}$ satisfying conditions (i) and (ii), we define
    \begin{equation}
        \omega_e:=\Lambda    
    \end{equation}
    and $ \omega\in\Omega^1(G,\mathfrak{h})$ by
    \begin{equation}
        \omega_g:=\left(L_{g^{-1}}\right)^* \omega_e = \left(L_{g^{-1}}\right)^*\Lambda.
    \end{equation}
    For any $a,g\in G$ and $X\in T_gG$, we have
    \begin{eqnarray*}
        \left((L_a)^*\omega\right)_g(X) &=& \omega_{ag}\left((L_a)_*X\right) = \omega_{e}\left(\left(L_{(ag)^{-1}}\right)_*(L_a)_*X\right)\\
        &=& \omega_{e}\left(\left(L_{g^{-1}}\right)_*X\right) = \omega_g(X).
    \end{eqnarray*}
    which shows that $(L_a)^*\omega=\omega$ for all $a\in G$, i.e., $\omega$ is $G$-invariant. Given $A\in\mathfrak{h}$ and $u\in G$, the fundamental vector field $A^\#(u)$ is 
    \begin{equation}
        A^\#(u)=\left.\frac{d}{dt}u\,\exp(tA)\right|_{t=0}=(L_u)_*A.
    \end{equation}
    This implies
    \begin{equation}
        \omega_u(A^\#(u))=\omega_u((L_u)_*A)=\omega_e(A)=\Lambda(A)=A,
    \end{equation}
    with the last equality coming from condition (ii), which shows that $\omega$ satisfies condition (i) in \cref{def:connection_1_form}. Furthermore, with $g\in G$, $h\in H$, and $X\in T_gG$, we have
    \begin{eqnarray*}
        \left(\left(R_h\right)^*\omega\right)_g(X) &=& \omega_{gh}\left((R_h)_*X\right) = \omega_e\left(\left(L_{(gh)^{-1}}\right)_*(R_h)_*X\right)\\
        &=& \Lambda(\left(L_{h^{-1}}\right)_*\left(L_{g^{-1}}\right)_*(R_h)_*X) = \Lambda(\left(L_{h^{-1}}\right)_*(R_h)_*\left(L_{g^{-1}}\right)_*X)\\
        &=& \Lambda(\mathrm{Ad}_{h^{-1}}\circ\left(L_{g^{-1}}\right)_*X) = \mathrm{Ad}_{h^{-1}}\circ\Lambda(\left(L_{g^{-1}}\right)_*X)\\
        &=& \mathrm{Ad}_{h^{-1}}\circ\omega_g(X).
    \end{eqnarray*}
    where we have used condition (i). Consequently, $\omega$ also satisfies condition (ii) in~\cref{def:connection_1_form}, making it a $G$-invariant Ehresmann connection on $G$.
    
    Conversely, given a $G$-invariant Ehresmann connection $\omega \in \Omega^1(G,\mathfrak{h})$ we define a map $\Lambda:\mathfrak{g}\to \mathfrak{h}$ by
    \begin{equation}
        \Lambda:=\omega_e.
    \end{equation}
    Linearity of $\Lambda$ follows from the linearity of $\omega$. We note that $\omega_g(X) = \omega_e((L_{g^{-1}})_*X)$ for all $g\in G$ and $X\in T_gG$ since $\omega$ is $G$-invariant. Combined with condition (ii) in \cref{def:connection_1_form}, this implies
    \begin{eqnarray*}
        \mathrm{Ad}_h\circ\Lambda(X) &=& \mathrm{Ad}_h\circ\omega_e(X) = ((R_{h^{-1}})^*\omega)_e(X) = \omega_{h^{-1}}((R_{h^{-1}})_*X)\\
        &=& \omega_{e}((L_h)_*(R_{h^{-1}})_*X) = \omega_{e}(\mathrm{Ad}_hX) = \Lambda\circ\mathrm{Ad}_h(X) \,,
    \end{eqnarray*}
    for any $h\in H$ and $X\in\mathfrak{g}$, i.e., $\Lambda$ satisfies (i). Furthermore, for any $A \in H$, we have $A=A^\#(e)$ and
    \begin{equation}
        \Lambda(A)=\omega_e(A)=\omega_e(A^\#(e))=A.
    \end{equation}
    which implies that $\Lambda$ also satisfies condition (ii).

    Finally, we verify that the correspondence described above is one-to-one. Given the linear map $\Lambda:\mathfrak{g}\to\mathfrak{h}$ satisfying conditions (i) and (ii), we constructed $\omega$ by $\omega_g:=(L_{g^{-1}})^*\Lambda$. In particular, $\Lambda=(L_{e^{-1}})^*\Lambda=\omega_e$. Conversely, if $\omega$ is a $G$-invariant connection and we define $\Lambda:=\omega_e$, then by $G$-invariance
    \begin{equation}
        \omega_g(X) = \omega_e((L_{g^{-1}})_*X) = \Lambda((L_{g^{-1}})_*X) = (L_{g^{-1}})^*\Lambda(X) \,,
    \end{equation}
    for all $g\in G$ and $X\in T_gG$. We conclude that the correspondence of the two constructions is one-to-one.
\end{proof}
\end{document}